%% file: icml2024.tex
\documentclass{article}

\usepackage{microtype}
\usepackage{graphicx}
\usepackage{booktabs} %

\usepackage[pagebackref]{hyperref}

\usepackage[accepted]{icml2024}

\usepackage{amsmath}
\usepackage{amssymb}
\usepackage{mathtools}
\usepackage{amsthm}

\theoremstyle{plain}
\newtheorem{theorem}{Theorem}[section]
\newtheorem{proposition}[theorem]{Proposition}
\newtheorem{lemma}[theorem]{Lemma}

\theoremstyle{definition}

\theoremstyle{remark}
\newtheorem{remark}[theorem]{Remark}

\usepackage[textsize=tiny]{todonotes}

\PassOptionsToPackage{dvipsnames}{xcolor}
\input{preamble}

\forreviewfalse

\icmltitlerunning{Operator SVD with Neural Networks via Nested Low-Rank Approximation}

\begin{document}

\twocolumn[
\icmltitle{Operator SVD with Neural Networks via Nested Low-Rank Approximation}

\icmlsetsymbol{equal}{*}

\begin{icmlauthorlist}
\icmlauthor{J. Jon Ryu}{mit}
\icmlauthor{Xiangxiang Xu}{mit}
\icmlauthor{H. S. Melihcan Erol}{mit}
\icmlauthor{Yuheng Bu}{uf}
\icmlauthor{Lizhong Zheng}{mit}
\icmlauthor{Gregory W. Wornell}{mit}
\end{icmlauthorlist}

\icmlaffiliation{mit}{Department of EECS, MIT, Cambridge, Massachusetts, USA}
\icmlaffiliation{uf}{Department of ECE, University of Florida, Gainesville, Florida, USA}

\icmlcorrespondingauthor{J. Jon Ryu}{\href{jongha@mit.edu}{jongha@mit.edu}}

\icmlkeywords{neural computing, singular-value decomposition, low-rank approximation, partial differential equations, representation learning}

\vskip 0.3in
]

\printAffiliationsAndNotice{}  %

\begin{abstract}
\input{abstract}
\end{abstract}

\input{main_text}

\newpage
\section*{Acknowledgements}
The authors appreciate insightful feedback and helpful suggestions from anonymous
reviewers to improve earlier versions of the manuscript.
The authors also acknowledge the MIT SuperCloud and Lincoln Laboratory Supercomputing Center for providing HPC resources that have contributed to the research results reported within this paper.
JJR and GWW were supported, in part, by the MIT-IBM Watson AI Lab under Agreement No. W1771646, and by MIT Lincoln Laboratory.
HSME, XX, and LZ were supported by the Office of Naval Research (ONR) under grant N00014-19-1-2621.

\section*{Impact Statement}
This paper presents work whose goal is to advance the field of Machine Learning. There are many potential societal consequences of our work, none of which we feel must be specifically highlighted here.

\bibliographystyle{icml2024}
\bibliography{ref}

\newpage
\appendix
\onecolumn
\addtocontents{toc}{\protect\StartAppendixEntries}
\listofatoc
\input{appendix}

\end{document}

%% file: preamble.tex
\usepackage[T1]{fontenc}    %
\usepackage{makecell}
\usepackage[labelfont=it,font+=smaller]{caption}
\usepackage[labelfont=it]{subcaption}
\usepackage{enumitem}
\usepackage[export]{adjustbox}

\usepackage{amsmath}
\usepackage{amssymb}
\usepackage{mathtools}
\usepackage{amsthm}
\usepackage{thmtools}
\usepackage{thm-restate}

\input{defns}

\input{defns_local}

\usepackage[dvipsnames]{xcolor}

\usepackage[scale=0.82]{FiraMono}

\usepackage{multirow}

\usepackage{empheq}

\usepackage{listings}
\usepackage{pifont}%
\newcommand{\cmark}{\textcolor{blue}{\ding{51}}}%
\newcommand{\xmark}{\textcolor{red}{\ding{55}}}%

\usepackage{transparent}
\definecolor{codegreen}{rgb}{0,0.6,0}
\definecolor{codegray}{rgb}{0.5,0.5,0.5}
\definecolor{codepurple}{rgb}{0.58,0,0.82}
\definecolor{backcolour}{rgb}{0.95,0.95,0.92}

\lstdefinestyle{mystyle}{
    language=Python,
    keywords={def,class,return,None,raise,mean,sqrt},    
    backgroundcolor=\color{backcolour},   
    commentstyle=\color{codegreen},
    keywordstyle=\color{blue},
    numberstyle=\tiny\color{codegray},
    stringstyle=\color{codepurple},
    basicstyle=\ttfamily\footnotesize,
    breakatwhitespace=false,         
    breaklines=true,                 
    captionpos=b,                    
    keepspaces=true,                 
    numbers=left,                    
    xleftmargin=.375cm,
    numbersep=5pt,
    showspaces=false,                
    showstringspaces=false,
    showtabs=false,                  
    tabsize=4,    
}
\usepackage{authblk}

\usepackage{multicol}
\usepackage[pagebackref]{hyperref}

\usepackage{wrapfig}

\newif\ifforreview

\usepackage{scrwfile}
\TOCclone[\appendixname]{toc}{atoc}
\newcommand\StartAppendixEntries{}
\AfterTOCHead[toc]{%
  \renewcommand\StartAppendixEntries{\value{tocdepth}=-10000\relax}%
}
\AfterTOCHead[atoc]{%
  \edef\maintocdepth{\the\value{tocdepth}}%
  \value{tocdepth}=-10000\relax%
  \renewcommand\StartAppendixEntries{\value{tocdepth}=\maintocdepth\relax}%
}

\renewcommand{\E}{\mathbb{E}}

\renewcommand{\S}{Sec.~}

%% file: defns.tex
\usepackage{xspace}
\usepackage{bbm}
\input{mathlig}
\usepackage{mathtools}
\usepackage{relsize}
\usepackage{mathrsfs}
\usepackage{dsfont}
\DeclarePairedDelimiterX{\inp}[2]{\langle}{\rangle}{#1, #2}
\makeatletter
\newcommand*\bigcdot{\mathpalette\bigcdot@{.5}}
\newcommand*\bigcdot@[2]{\mathbin{\vcenter{\hbox{\scalebox{#2}{$\m@th#1\bullet$}}}}}
\makeatother

\newcommand{\muspace}{\mspace{1mu}}

\DeclareRobustCommand{\scond}{\mathchoice{\muspace\vert\muspace}{\vert}{\vert}{\vert}}
\mathlig{|}{\scond}

\DeclareRobustCommand{\discint}{\mathchoice{\mspace{-1.5mu}:\mspace{-1.5mu}}{\mspace{-1.5mu}:\mspace{-1.5mu}}{:}{:}}
\mathlig{::}{\discint}
\newcommand{\suchthat}{\mathchoice{\colon}{\colon}{:\mspace{1mu}}{:}}

\def\tr{\mathop{\rm tr}\nolimits}%
\def\diag{\mathop{\rm diag}\nolimits}%

\newcommand{\Fc}{\mathcal{F}}
\newcommand{\Gc}{\mathcal{G}}
\newcommand{\Hc}{\mathcal{H}}
\newcommand{\Ic}{\mathcal{I}}

\newcommand{\Lc}{\mathcal{L}}

\newcommand{\Nc}{\mathcal{N}}

\newcommand{\Tc}{\mathcal{T}}

\newcommand{\Xc}{\mathcal{X}}

\newcommand{\fv}{{\bf f}}
\newcommand{\gv}{{\bf g}}

\newcommand{\xv}{{\bf x}}
\newcommand{\yv}{{\bf y}}

\newcommand{\uv}{{\bf u}}
\newcommand{\vv}{{\bf v}}

\newcommand{\ph}{{\hat{p}}}

\newcommand{\eb}{{\mathbf e}}

\newcommand{\xb}{{\mathbf x}}

\def\b{\beta}

\def\d{\delta}

\def\eps{\epsilon}
\def\th{\theta}

\DeclareMathOperator\E{\mathsf{E}}

\newcommand\eg{e.g.,\xspace}
\newcommand\ie{i.e.,\xspace}
\def\textiid{i.i.d.\@\xspace}
\newcommand\iid{\ifmmode\text{ i.i.d. } \else \textiid \fi}

\newcommand{\Real}{\mathbb{R}}
\newcommand{\Natural}{\mathbb{N}}

\newcommand{\ones}{\mathds{1}}

\newcommand{\half}{\frac{1}{2}}%

\def\mathllap{\mathpalette\mathllapinternal}
\def\mathllapinternal#1#2{%
  \llap{$\mathsurround=0pt#1{#2}$}}

\def\clap#1{\hbox to 0pt{\hss#1\hss}}
\def\mathclap{\mathpalette\mathclapinternal}
\def\mathclapinternal#1#2{%
  \clap{$\mathsurround=0pt#1{#2}$}}

\let\oldstackrel\stackrel
\renewcommand{\stackrel}[2]{\oldstackrel{\mathclap{#1}}{#2}}

\DeclarePairedDelimiterX{\infdivx}[2]{(}{)}{%
  #1\;\delimsize\|\;#2%
}

\renewcommand{\hbar}{h\mathllap{\overline{\vphantom{h}\hphantom{\rule{4.6pt}{0pt}}}\mspace{0.77mu}}}

\catcode`~=11 %
\newcommand{\urltilde}{\kern -.06em\lower -.06em\hbox{~}\kern .02em}
\catcode`~=13 %

\DeclareMathOperator*{\argmin}{arg\,min}

\DeclarePairedDelimiterX{\norm}[1]{\lVert}{\rVert}{#1}
\DeclarePairedDelimiterX{\abs}[1]{\lvert}{\rvert}{#1}

\usepackage{xparse}

\newcommand*\diff{\mathop{}\!\mathrm{d}}

\let\oldpartial\partial
\renewcommand*{\partial}{\mathop{}\!\oldpartial}

\newcommand{\defeq}{\mathrel{\mathop{:}}=}

%% file: defns_local.tex
\newcommand{\ft}{\tilde{f}}

\newcommand{\kbar}{\overline{k}}

\newcommand{\Imatrix}{\mathsf{I}}

\newcommand{\Km}{\mathsf{K}}

\newcommand{\Lm}{\mathsf{L}}

\newcommand{\Lr}{\mathscr{L}}

\newcommand{\weights}{\boldsymbol{\mathsf{w}}}
\newcommand{\Bm}{\mathsf{B}}
\newcommand{\Pm}{\mathsf{P}}
\newcommand{\Um}{\mathsf{U}}
\newcommand{\Vm}{\mathsf{V}}

\newcommand{\phit}{\tilde{\phi}}
\newcommand{\psit}{\tilde{\psi}}

\newcommand{\phib}{\boldsymbol{\phi}}
\newcommand{\psib}{\boldsymbol{\psi}}

\newcommand{\vectormask}{\boldsymbol{\mathsf{m}}}
\newcommand{\matrixmask}{{\mathsf{M}}}
\newcommand{\Mm}{\matrixmask}
\newcommand{\bast}{\boldsymbol{\ast}}

\newcommand{\sg}{\mathsf{sg}}

\newcommand{\phiv}{\boldsymbol{\phi}}

\newcommand{\Jb}{\mathbf{J}}

\DeclareMathOperator*{\maximize}{maximize}

\newcommand{\Am}{\mathsf{A}}

\newcommand{\rk}{r}
\newcommand{\loraobj}{\Lc_{\mathsf{LoRA}}}
\newcommand{\jntloraobj}{\Lc_{\mathsf{jnt}}}

\DeclareMathAlphabet{\mathpzc}{OT1}{pzc}{m}{it}

\newcommand{\Iop}{\mathpzc{I}}
\newcommand{\Kop}{\mathpzc{K}}
\newcommand{\Top}{\mathpzc{T}}
\newcommand{\Hop}{\mathpzc{H}}
\newcommand{\Vop}{\mathpzc{V}}

\newcommand{\redt}[1]{#1}
\newcommand{\bluet}[1]{#1}

%% file: abstract.tex
Computing eigenvalue decomposition (EVD) of a given linear operator, or finding its leading eigenvalues and eigenfunctions,
is a fundamental task in many machine learning and scientific computing problems.
For high-dimensional eigenvalue problems, training neural networks to parameterize the eigenfunctions is considered as a promising alternative to the classical numerical linear algebra techniques. 
This paper proposes a new optimization framework based on the low-rank approximation characterization of a truncated singular value decomposition, accompanied by new techniques called \emph{nesting} for learning the top-$L$ singular values and singular functions in the correct order. 
The proposed method promotes the desired orthogonality in the learned functions implicitly and efficiently via an unconstrained optimization formulation, which is easy to solve with off-the-shelf gradient-based optimization algorithms.
We demonstrate the effectiveness of the proposed optimization framework for use cases in computational physics and machine learning.

%% file: main_text.tex
\allowdisplaybreaks

\newcommand{\gradopt}{\mathsf{GradOpt}}
\newcommand{\nsvdjnt}{NeuralSVD$_{\text{jnt}}$}
\newcommand{\nsvdseq}{NeuralSVD$_{\text{seq}}$}
\newcommand{\ellthloraobj}{\Lc_{\ell}}
\newcommand{\mm}{\mathsf{m}}

\section{Introduction}
Spectral decomposition techniques, including singular value decomposition (SVD) and eigenvalue decomposition (EVD), are crucial tools in machine learning and data science for handling large datasets and reducing their dimensionality while preserving prominent structures; see, \eg \citep{Markovsky2012,Blum--Hopcroft--Kannan2020}. 
They break down a matrix (or a linear operator) into its constituent parts, enabling a better understanding of the underlying geometry and relationships within the data.
These form the foundation of various low-dimensional embedding algorithms~\citep{Scholkopf--Smola--Muller1998,Tenenbaum--De-Silva--Langford2000,Roweis--Saul2000,Shi--Malik2000,Ng--Jordan--Weiss2001,Belkin--Niyogi2003,Bengio--Vincent--Paiment--Delalleau--Ouimet--Le-Roux2003,Cox--Cox2008} and correlation analysis algorithms~\citep{Michaeli--Wang--Livescu2016,Wang--Wu--Huang--Zheng--Xu--Zhang--Huang2019} and are widely used in image and signal processing~\citep{Andrews--Patterson1976,Turk--Pentland1991,Wiskott--Sejnowski2002,Sprekeler2011,Scetbon--Elad--Milanfar2021}, natural language processing~\citep{Landauer--Foltz--Laham1998,Goldberg--Levy2014}, among other fields.
Beyond machine learning applications, solving eigenvalue problems is a crucial step in solving partial differential equations (PDEs), such as Schr\"odinger's equations in quantum chemistry~\citep{Hermann--Schatzle--Noe2020,Pfau--Spencer--Matthews2020}.

The standard approach to these problems in practice is to perform the matrix spectral decomposition using the standard techniques from numerical linear algebra~\citep{Golub--Van-Loan2013}.
In machine learning, the size of the matrix is typically given by the size of the data sample or the dimensionality of data. In physical simulation, the underlying matrix scales with the resolution of discretization of a given domain.
For finding a few top (or bottom) eigenmodes, in general, iterative subspace methods such as Krylov subspace methods~\citep{Saad1981} and LOBPCG~\citep{Knyazev2001} can efficiently find top eigenmodes via repeating matrix-vector products~\citep{Golub--Van-Loan2013}.
Note that the full eigendecomposition of a $N\times N$ matrix can be performed in $O(N^3)$ time complexity if the matrix can be stored in memory. 
For large-scale, high-dimensional data, however,
the memory, computational, and statistical complexity of matrix decomposition algorithms
poses a significant challenge in practice.
As the data size (or the resolution of the grid in physical simulation) or the dimensionality of the underlying problem increases, the matrix-based approach becomes easily infeasible as even storing the eigenvectors in memory is too costly.

A promising alternative is to approximate the singular- or eigen-functions using parametric function approximators, 
assuming that there exists an abstract operator that induces a target matrix to decompose.
In other words, we aim to approximate an eigenvector $\hat{\phiv}_\ell\in\Real^N$ by a single parametric function $\hat{\phi}_\ell\suchthat\Xc\to\Real$.
In Fig.~\ref{fig:schematic}, we illustrate the proposed framework NeuralSVD, which is a special instance of the parametric approach, as a schematic diagram. 

Compared to the ``nonparametric'' approach, the parametric approach has several advantages.
First, unlike the nonparametric approach which relies on the Nystr\"om method~\citep{Williams--Seeger2000,Bengio--Delalleau--Le-Roux--Paiement--Vincent--Ouimet2004} to extrapolate eigenvectors to unseen points, the parametric eigenfunctions can naturally extrapolate without the storage and computational complexity of Nystr\"om; refer to \S\ref{app:sec:nystrom} for a detailed discussion.
Second, given the exceptional ability of neural networks (NNs) to generalize with complex data, such as convolutional neural networks for images, transformers for natural language, and recently developed NN ansatzes for quantum chemistry~\citep{Hermann--Schatzle--Noe2020,Pfau--Spencer--Matthews2020}, one can anticipate better extrapolation performance than in the nonparametric, matrix approach.
If the choice of parametric functions is appropriate to exploit the complex structure of the underlying distribution, we can also expect the parametric approach to scale better in terms of training complexity for large-scale problems than the nonparametric counterpart.
Third, in the context of solving PDEs, the parametric approach stands out, notably because it necessitates only a sampler from a specified domain without the need for discretization. This is particularly advantageous as it helps mitigate the potential introduction of undesirable approximation errors.

In this paper, we propose a new optimization framework that can train neural networks to approximate the top-$L$ orthogonal singular- (or eigen-) functions of an operator.
The proposed method is based on an unconstrained optimization problem from Schmidt's low-rank approximation theorem~(\citeyear{Schmidt1907}) 
that naturally admits an unbiased gradient estimator. 
To learn the ordered top-$L$ orthogonal singular basis as the optimal solution simultaneously, %
we introduce new techniques called \emph{nesting} to break the symmetry so that we can learn the singular functions in the order of singular values; see the high-level illustrations in Fig.~\ref{fig:schematic_nesting}.

\begin{figure}[t!]
    \centering
    \includegraphics[width=.45\textwidth,valign=t]{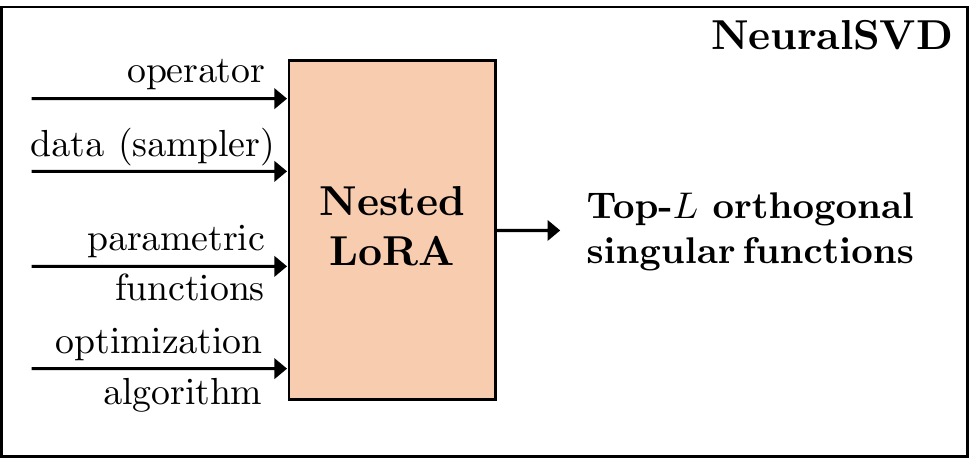}
    \caption{Schematic illustration of NeuralSVD.}\vspace{-1em}
    \label{fig:schematic}
\end{figure}

While several frameworks have been proposed in the machine learning community to systematically recover ordered eigenfunctions using neural networks~\citep{Pfau--Petersen--Agarwal--Barrett--Stachenfeld2019,Deng--Shi--Zhu2022}, these approaches encounter practical optimization challenges, particularly in enforcing the orthonormality of the learned eigenfunctions.
Compared to the prior works, our framework can (1) learn the top-$L$ orthogonal singular bases more efficiently for larger $L$ due to the more stable optimization procedure, and (2) perform SVD of a non-self-adjoint operator by design, handling EVD of a self-adjoint operator as a special case.
We demonstrate the power of our framework in solving PDEs and representation learning for cross-domain retrieval.

\section{Problem Setting and Preliminaries}
\label{sec:prelim}

\subsection{Operator SVD}
While SVD is typically assumed to be done via EVD, our low-rank approximation framework can directly perform SVD, handling EVD as a special case.
We consider two separable Hilbert spaces $\Fc$ and $\Gc$ and a linear operator $\Top\suchthat \Fc\to \Gc$.
We will use the bra-ket notation, which denotes $|f\rangle$ for a function $f(\cdot)$ throughout, as it allows us to describe the proposed method in a succinct way.
For most applications, the Hilbert spaces $\Fc$ and $\Gc$ are $\Lr^2$ spaces of square-integrable functions, and a reader thus can read the inner product between two real-valued functions $\langle f| f'\rangle$ as an integral $\int_{\Xc} f(x)f'(x)\mu(\diff x)$ for some underlying measure $\mu$ over a domain $\Xc$.
In learning problems, $\Top$ is typically an integral kernel operator induced by a kernel function, accompanied by data distributions as the underlying measures. 
In solving PDEs, $\Top$ is given as a differential operator that governs a physical system of interest, where the underlying measure is the Lebesgue measure over a domain.

For a \emph{compact} operator $\Top$, it is well known that there exist orthonormal bases $(\phi_i)_{i\ge 1}$ and $(\psi_i)_{i\ge 1}$ with a sequence of non-increasing, non-negative real numbers $(\sigma_i)_{i\ge 1}$ such that 
$(\Top \phi_i)(y) = \sigma_i \psi_i(y)$,
$(\Top^* \psi_i)(x) = \sigma_i \phi_i(x)$, $i=1,2,\ldots$.\footnote{Compact operators can be informally understood as a benign class of possibly infinite-dimensional operators that behave similarly to finite-dimensional matrices, so that we can consider the notion of SVD as in matrices. A formal definition is not crucial in understanding the manuscript and is thus deferred to \S\ref{app:sec:lora_derivation}.}
The function pairs $(\phi_i,\psi_i)$ are called (left- and right-, resp.) singular functions corresponding to the singular value $\sigma_i$. 
Hence, the compact operator $\Top$ can be written as
\begin{align}
\label{eq:svd}
\Top = \sum_{i=1}^{\infty} \sigma_i | \psi_i\rangle \langle\phi_i|,
\end{align}
for $\sigma_1\ge \sigma_2\ge \ldots \ge 0$,
which we call the SVD of $\Top$.
Here, $|\psi\rangle \langle \phi| \suchthat \Fc\to \Gc$ is the operator defined as $(|\psi\rangle \langle \phi|) |f\rangle \defeq (\langle \phi| f\rangle) |\psi\rangle$, which can be understood as the \emph{outer product}.
We refer an interested reader to \citep[Theorem 7.6]{Weidmann2012} for a rigorous treatment of SVD of compact operators.

\subsection{EVD as a Special Case of SVD}
In several applications, the operator is self-adjoint (\ie $\Top^*=\Top$ with $\Fc=\Gc$), and sometimes even positive definite (PD).
By the spectral theorem~\citep[Theorem 7.1]{Weidmann2012}, a compact self-adjoint operator has the EVD of the form $\Top=\sum_{i=1}^{\infty} \lambda_i |\phi_i\rangle\langle\phi_i|$.
In this case, the singular values of the operator are the absolute values of its eigenvalues, and
for each $i$, the $i$-th left- and right- singular functions are either identical (if $\lambda_i \ge 0$) or only different by the sign (if $\lambda_i <0$).
Hence, in particular, we can find its EVD by SVD in the case of a positive-definite (PD) operator.
We remark in passing that our framework is also applicable for a certain class of non-compact operators; see \S\ref{sec:exp:pde} and \S\ref{app:sec:noncompact}.

\subsection{SpIN and NeuralEF}
As alluded to earlier, Spectral Inference Networks (SpIN)~\citep{Pfau--Petersen--Agarwal--Barrett--Stachenfeld2019} and Neural Eigenfunctions (NeuralEF)~\citep{Deng--Shi--Zhu2022} are the most closely related prior works to ours, in the sense that these methods aim to \emph{learn} the top-$L$ orthonormal eigenbasis of a self-adjoint operator by \emph{training} parametric functions.
Though there exist other approaches in computational physics that aim to find beyond the top mode or ground state, most, if not all, approaches are based on rather ad-hoc regularization terms and do not have guarantee to recover the top-$L$ ordered eigenfunctions.
Hence, we briefly overview SpIN and NeuralEF in the main text, and discuss the two methods in greater details as well as the other line of works in \S\ref{app:sec:related_work}.

SpIN and NeuralEF are only applicable for self-adjoint operators, and thus we temporarily assume a self-adjoint operator $\Top\suchthat\Fc\to\Fc$ in the rest of this section.
SpIN and NeuralEF are both grounded in the principle of maximizing the Rayleigh quotient with orthonormality constraints. However, their optimization frameworks encounter nontrivial complexity issues, as summarized in Table~\ref{tab:comparison}. The primary challenge lies in efficiently handling these orthonormality constraints. To achieve fast convergence with off-the-shelf gradient-based optimization algorithms, it is also crucial to estimate gradients in an unbiased manner. 

SpIN starts from the following variational characterization of the top-$L$ orthonormal eigenbasis:
\begin{align}
\label{eq:evd_char_subspace}
\begin{aligned}
\underset{\substack{\phit_\ell\in \Fc,\, \ell\in[L]}}{\text{maximize~~}} & 
\langle \phit_\ell | \Top\phit_\ell\rangle \\
\text{subject~to~~} &
\langle \phit_i | \phit_{i'}\rangle=\d_{ii'}
\quad \forall 1\le i,i'\le L.
\end{aligned}
\end{align}
Since this formulation only captures the subspace without order, SpIN employs a special gradient masking scheme to learn the eigenfunctions in the correct order.
The resulting algorithm involves Cholesky decomposition of $L\times L$ matrix per iteration, which takes $O(L^3)$ complexity in general.
Further, to come up with unbiased gradient estimates, SpIN introduces a hyperparameter-sensitive bi-level optimization and necessitates the need to store the Jacobian of the parametric model.
As a result, the unfavorable scalability with $L$, along with memory complexity and implementation challenges, reduces the practical utility of SpIN.

To circumvent the issues with SpIN, NeuralEF adopted and extended an optimization framework of EigenGame~\citep{Gemp--McWilliams--Vernade--Graepel2021}, which is a game-theoretic formulation for streaming PCA. The underlying optimization problem can be understood as a variant of the \emph{sequential} version of the subspace characterization~\eqref{eq:evd_char_subspace};
see \S\ref{app:sec:neuralef} and \eqref{eq:neuralef_problem} therein. 
The resulting optimization, however, still suffers from its biased gradient estimation, and requires the parametric functions to be normalized, \ie  $\|\hat{\phi}_\ell\|_2=\langle\hat{\phi}_\ell|\hat{\phi}_\ell\rangle^{1/2}=1$ for every $\ell$. While the biased gradient could be alleviated via its simple variant as we explain in \S\ref{app:sec:neuralef}, our experiments show that the function normalization step may slow down the convergence in practice.

We note that both SpIN and NeuralEF require each parametric eigenfunction to be parameterized separately, \ie without shared parameters among them, to ensure that their optimization schemes work. In practice, while using disjoint models is a straightforward choice, it may consume excessive memory if the number of modes 
$L$ to be retrieved is large or if the model becomes more complex. To address both scenarios, in the next section, we provide two techniques: one suitable for disjoint parameterization (\S\ref{sec:seq_nesting}) and the other for joint parameterization (\S\ref{sec:jnt_nesting}).

\section{SVD via Nested Low-Rank Approximation}
\label{sec:theory}
In what follows, we propose a new optimization-based algorithm for SVD with neural networks, based on Schmidt's approximation theorem combined with new techniques called \emph{nesting} for learning the singular functions in order.
The resulting framework is significantly conceptually simpler and easier to implement than prior methods, without introducing sophisticated optimization techniques.
Further, unlike SpIN and NeuralEF, we can directly perform the SVD of a non-self-adjoint operator.
Hereafter, we assume that $\Top$ has $\{(\sigma_{\ell}, |\phi_\ell\rangle, |\psi_\ell\rangle)\}_{\ell=1}^{\infty}$ as its orthonormal singular triplets. 
\subsection{Learning Subspaces via Low-Rank Approximation}
Let $L$ be the number of modes we wish to retrieve.
We will use a shorthand notation $\fv_{1:\ell}(x)\defeq [f_1(x),\ldots,f_\ell(x)]^\intercal$.
Below, we will employ distinct variables $|f\rangle$ and $|g\rangle$ as counterparts to $|\phi\rangle$ and $|\psi\rangle$, respectively, which represent normalized singular functions. The intentional use of separate variables $|f\rangle$ and $|g\rangle$ underscores their role in representing \emph{scaled} singular functions rather than normalized ones within our framework. The importance of this distinction will become apparent in the following subsection.

For the top-$L$ SVD of a given operator $\Top$, we consider the \emph{low-rank approximation} (LoRA) objective defined as
\begin{align}
&\loraobj(\fv_{1:L}, \gv_{1:L})\defeq \loraobj(\fv_{1:L}, \gv_{1:L};\Tc)
\label{eq:lora_obj}\\
&\defeq -2\sum_{\ell=1}^L \langle g_\ell|\Top f_\ell\rangle + \sum_{\ell=1}^L\sum_{\ell'=1}^L \langle f_\ell|f_{\ell'}\rangle \langle g_\ell|g_{\ell'}\rangle.\nonumber
\end{align}
This objective can be derived as the approximation error of $\Top$ via a low-rank expansion $\sum_{\ell=1}^L |f_\ell\rangle \langle g_\ell|$ measured in the squared Hilbert--Schmidt norm, for a compact operator $\Top$.
We defer its derivation to \S\ref{app:sec:lora_derivation}.
By Schmidt's LoRA theorem~\citep{Schmidt1907}, which is the operator counterpart of \citet{Eckart--Young1936} for matrices, 
$(\fv^\star,\gv^\star)$ corresponds to the rank-$L$ approximation of $\Top$.
The proof of the following theorem can be found in \S\ref{app:sec:proof_thm:subspace}.

\begin{theorem}
\label{thm:subspace}
Assume that $\Top\suchthat\Fc\to\Gc$ is compact. Let
$((f_{\ell}^\star,g_{\ell}^\star))_{\ell=1}^L\in(\Fc\times\Gc)^L$ be a global minimizer of 
$\loraobj(\fv_{1:L}, \gv_{1:L})$.
If $\sigma_L >\sigma_{L+1}$, then
\[
\sum_{\ell=1}^L |g_\ell^\star\rangle \langle f_\ell^\star|
=\sum_{\ell=1}^L \sigma_\ell |\psi_\ell\rangle \langle\phi_\ell|.
\]
\end{theorem}
\newcommand{\Qm}{\mathsf{Q}}
In cases of degeneracy, \ie when multiple singular functions share the same singular value, a minimizer will still recover a subspace spanned by the singular functions associated with that particular singular value.
Throughout, we will assume such strict spectral gap assumptions for the sake of simple exposition.

\subsection{Nesting for Learning Ordered Singular Functions}\label{sec:nesting}
While the LoRA characterization of the spectral subspaces is favorable in practice due to its unconstrained nature, a global minimizer only characterizes the top-$L$ singular \emph{subspaces}; note that $(\Qm\fv^\star, \Qm\gv^\star)$ for any orthogonal matrix $\Qm\in\Real^{L\times L}$ is also a global minimizer.
We thus require an additional technique to find the singular functions and singular values \textit{in order} by breaking the symmetry in the objective $\loraobj(\fv_{1:L},\gv_{1:L})$.

The idea for learning the ordered solution is as follows.
Suppose that we can find a common global minimizer $(\fv_{1:L}^\star, \gv_{1:L}^\star)$ of the objectives $\loraobj(\fv_{1:\ell},\gv_{1:\ell})$ for $1\le \ell\le L$.
Then, from the optimality in Theorem~\ref{thm:subspace}, $\sum_{i=1}^\ell |g_i^\star\rangle\langle f_{i}^\star|$ must be the rank-$\ell$ approximation of $\Top$ for each $\ell\in[L]$, which is $\sum_{i=1}^\ell \sigma_i |\psi_i\rangle\langle \phi_i|$.
By telescoping, we then have $|g_\ell^\star\rangle\langle f_{\ell}^\star| = \sigma_\ell |\psi_\ell\rangle\langle \phi_\ell|$ for each $\ell\in[L]$, which is the desired solution.
Since the optimization is performed with a certain nested structure, we call this idea \emph{nesting}.

We remark that, unlike most existing methods that aim to directly learn ortho-\emph{normal} eigenfunctions, 
the global optimum with (nested) LoRA characterizes the correct singular functions \emph{scaled} by the singular value $\sigma_\ell$, as alluded to earlier. Using this property, one can estimate $\sigma_\ell$ by computing the product of norms $\|f_\ell^\star\|\cdot \|g_\ell^\star\|$; see \S\ref{app:sec:norms} for the detail.

Below, we introduce two different versions that implement this idea: \emph{sequential nesting}, which is ideal when each eigenfunction is parameterized by disjoint neural networks, and \emph{joint nesting}, which can be used even when 
they may share parameters.

\subsubsection{Sequential Nesting}
\label{sec:seq_nesting}
Sequential nesting is based on the following observation: if $(\fv_{1:\ell-1},\gv_{1:\ell-1})$ already captures the top-$(\ell-1)$ singular subspaces as a minimizer of $\loraobj(\fv_{1:\ell-1},\gv_{1:\ell-1})$,
minimizing $\loraobj(\fv_{1:\ell},\gv_{1:\ell})$ for $(f_{\ell},g_{\ell})$ finds the $\ell$-th singular functions. Its proof can be found in \S\ref{app:sec:proof_thm:seq_nesting}. Formally:
\begin{theorem}
\label{thm:seq_nesting}
Assume that $\Top\suchthat\Fc\to\Gc$ is compact.
Pick any $\ell\ge 1$. 
Let $(\fv_{\ell}^\star,\gv_{\ell}^\star)\in\Fc\times\Gc$ be a global minimizer of 
$\loraobj(\fv_{1:\ell},\gv_{1:\ell})$, 
where $\sum_{i=1}^{\ell-1} |g_i\rangle \langle f_i|
= \sum_{i=1}^{\ell-1} \sigma_i |\psi_i\rangle \langle \phi_i|$.
If $\sigma_{\ell}>\sigma_{\ell+1}$, then 
$|g_{\ell}^\star\rangle \langle f_{\ell}^\star|
= \sigma_{\ell} |\psi_{\ell}\rangle \langle \phi_{\ell}|$.
\end{theorem}
We can implement this idea by simultaneously updating the iterate $(f_{\ell}^{(t)},g_{\ell}^{(t)})$ at time step $t\ge 1$ for each $\ell\in[L]$, to minimize $\loraobj(\fv_{1:\ell}^{(t)},\gv_{1:\ell}^{(t)})$, treating $(\fv_{1:\ell-1}^{(t)},\gv_{1:\ell-1}^{(t)})$ as a good proxy to the global optimum. That is, for each $\ell\in[L]$,
\begin{align}
\label{eq:seq_nesting}
&(f_{\ell}^{(t+1)},g_{\ell}^{(t+1)}) \\
&\gets \gradopt((f_{\ell}^{(t)},g_{\ell}^{(t)}), \partial_{\bluet{{(f_{\ell},g_{\ell})}}}
\loraobj(\fv_{1:\ell}^{(t)}, \gv_{1:\ell}^{(t)})).\nonumber
\end{align}
Here, $\gradopt(\boldsymbol{\th}, \gv)$ denotes a gradient-based optimization algorithm that returns the next iterate based on the current iterate $\boldsymbol{\th}$ and the gradient $\gv$.

Suppose that each model pair $(f_\ell,g_\ell)$ is parameterized via $L$ separate models with (disjoint) parameters $\th=(\th_\ell)_{\ell=1}^L$. 
In this case, the $\ell$-th eigenfunction can be updated independently from the $\ell'$-th eigenfunctions for $\ell'>\ell$ via the sequential nesting~\eqref{eq:seq_nesting}. 
Hence, while all $(\fv_{1:L},\gv_{1:L})$ are optimized simultaneously, the optimization is \emph{inductive} in the sense that the modes can be learned in the order of the singular values.
As a shorthand notation, let 
\[
\ellthloraobj\defeq \loraobj(\fv_{1:\ell},\gv_{1:\ell}).
\]
The gradient in \eqref{eq:seq_nesting} can be directly implemented by updating each $\th_\ell$ with the gradient 
\begin{align}
\partial_{\th_\ell} \ellthloraobj
&= \langle\partial_{\th_\ell} f_\ell| \partial_{f_\ell} \ellthloraobj\rangle 
+ \langle \partial_{\th_\ell} g_\ell| \partial_{g_\ell} \ellthloraobj\rangle,\nonumber\\
\text{where~~}|\partial_{f_\ell} \ellthloraobj\rangle
&=2\Bigl\{-|\Top^*g_\ell\rangle
+\sum_{i=1}^\ell |f_i\rangle \langle g_i| g_\ell\rangle  \Bigr\},
\label{eq:seq_nesting_grad} 
\end{align}
and $|\partial_{g_\ell} \ellthloraobj\rangle$ can be similarly computed by a symmetric expression.
Note that $|\partial_{\th_\ell} f_\ell\rangle$ should be understood as a vector-valued function of dimension $|\th_\ell|$, \ie the number of parameters in $\th_\ell$.
This gradient can be computed in a vectorized manner over $\ell\in [L]$.

\subsubsection{Joint Nesting}
\label{sec:jnt_nesting}
As alluded to earlier,
in the case of a shared parameterization, the sequential nesting~\eqref{eq:seq_nesting} may exhibit behavior that differs from its inductive nature with the shared parameterization.\footnote{\label{footnote:seq_nesting}We can still apply sequential nesting even when the functions are parameterized by a shared model; see \S\ref{app:sec:seq_shared} for a discussion.} 
For example, for a shared model, imperfect functions $(f_\ell,g_\ell)$ for some $\ell\in[L]$ may affect the already perfectly matched singular functions, say, $(f_1,g_1)$, unlike the disjoint parameterization case.

Interestingly, there is an alternative way to implement the idea of nesting that works for a shared parameterization with a guarantee.
The key observation is that the ordered singular values and functions $\{(\sigma_\ell,\phi_\ell,\psi_\ell)\}_{\ell=1}^L$ can be characterized as the global minimizer of a single objective function, by taking a weighted sum of $\{\ellthloraobj=\loraobj(\fv_{1:\ell},\gv_{1:\ell})\}_{\ell=1}^L$ with positive weights.
That is, define, for any positive weights $\weights=(w_1,\ldots,w_L)\in\Real_{>0}^L$,
\begin{align}
\jntloraobj(\fv_{1:L},\gv_{1:L};\weights)
\defeq \sum_{\ell=1}^L w_\ell \loraobj(\fv_{1:\ell},\gv_{1:\ell}).
\label{eq:jnt_nested_lora_obj}
\end{align}
\begin{theorem}
\label{thm:jnt_nesting}
Assume that $\Top\suchthat\Fc\to\Gc$ is compact. Let
$((f_{\ell}^\star,g_{\ell}^\star))_{\ell=1}^L\in(\Fc\times\Gc)^L$ be a global minimizer of 
$\jntloraobj(\fv_{1:L}, \gv_{1:L};\weights)$.
For any positive weights $\weights\in\Real_{>0}^L$, if the top-$(L+1)$ singular values are all distinct,
$|g_\ell^\star\rangle \langle f_\ell^\star|
= \sigma_\ell |\psi_\ell\rangle \langle \phi_\ell|$ for each $\ell\in[L]$.\footnote{Again, the strict spectral gap is assumed for simplicity; when there exist a degeneracy, the optimally learned functions should recover the orthonormal eigenbasis of the corresponding subspace.}
\end{theorem}
See \S\ref{app:sec:proof_thm:jnt_nesting} for its proof. 
The proof readily follows from observing that the joint objective $\jntloraobj(\fv_{1:\ell},\gv_{1:\ell})$ is minimized if and only if $\loraobj(\fv_{1:\ell},\gv_{1:\ell})$ is minimized for each $\ell\in[L]$, \ie $(\fv_{1:\ell},\gv_{1:\ell})$ characterizes the top-$\ell$ singular subspaces for each $\ell\in[L]$.
Any positive weights guarantee consistency, but we empirically found that the uniform weights $\weights=(\frac{1}{L},\ldots,\frac{1}{L})$ work well in practice.

Since the joint nesting is based on a single objective function~\eqref{eq:jnt_nested_lora_obj}, the optimization is as simple as
\begin{align}
&(\fv_{1:L}^{(t+1)},\gv_{1:L}^{(t+1)}) 
\label{eq:jnt_nesting}\\
&\gets \gradopt((\fv_{1:L}^{(t)},\gv_{1:L}^{(t)}), \partial_{({\fv_{1:L},\gv_{1:L}})}
\jntloraobj(\fv_{1:L}^{(t)},\gv_{1:L}^{(t)};\weights)).\nonumber
\end{align}
Even though the joint nesting can be implemented directly using an autograd package with \eqref{eq:jnt_nested_lora_obj}, overall training can be nearly twice as fast via manual gradient computation.
By the chain rule, the gradient can be computed as
\begin{align}
\partial_\th \jntloraobj
&= \sum_{\ell=1}^L \{\langle\partial_{\th} f_\ell | \partial_{f_\ell} \jntloraobj
\rangle + \langle \partial_{\th} g_\ell |\partial_{g_\ell} \jntloraobj
\rangle\},\nonumber\text{~where}\\
|\!\partial_{f_\ell} \jntloraobj
\rangle
&=2\Bigl\{-\mm_\ell|\Top^*g_\ell\rangle
\!+\!\sum_{i=1}^L \matrixmask_{i\ell} |f_i\rangle \langle g_i| g_\ell\rangle  \Bigr\}
\label{eq:jnt_nesting_grad}
\end{align}
and $|\partial_{g_\ell} \jntloraobj
\rangle$ is similarly computed.
Here, we define the vector mask as
$\mm_\ell \defeq \sum_{i=\ell}^L w_i$ and the matrix mask as $\mm_{\ell\ell'}=\mm_{\max\{\ell,\ell'\}}$; see \S\ref{sec:app:oneshot} for a formal derivation.
Lastly, setting $\mm_\ell\gets 1$ and $\mm_{i\ell}\gets \ones\{i\le \ell\}$ in \eqref{eq:jnt_nesting_grad} recovers the sequential nesting gradient~\eqref{eq:seq_nesting_grad}. 
Therefore, both versions of nesting can be implemented in a unified way via \eqref{eq:jnt_nesting_grad}.

\begin{remark}[Comparison to sequential nesting]
\label{rem:which_version_nesting}
In general, joint nesting may be less effective than sequential nesting with disjoint parameterization, as learning the top modes is affected by badly initialized latter modes, potentially slowing down the convergence. This is empirically demonstrated in \S\ref{sec:exp:pde}.
For the case of joint parameterization, however, we also empirically observe that joint nesting can outperform sequential nesting, as expected; see \S\ref{sec:exp:sketchy}.
Hence, we suggest users choose the version of nesting depending on the form of parameterization.
We provide an additional remark in \S\ref{sec:discussion}.
\end{remark}

\subsection{NeuralSVD: Nested LoRA with Neural Networks}
When combined with NN eigenfunctions, we call the overall approach \emph{\nsvdseq{}} and \emph{\nsvdjnt{}} based on the version of nesting, or \emph{NeuralSVD} for simplicity.
While the parametric approach can work with any parametric functions, we adopt the term \emph{neural} given that NNs represent a predominant class of powerful parametric functions.

In practice, we will need to use minibatch samples for optimization.
We explain how to implement the gradient updates of NestedLoRA based on the expression~\eqref{eq:jnt_nesting_grad} in a greater detail in \S\ref{app:sec:pseudo} with PyTorch code snippets.
\ifforreview
A PyTorch implementation of our method can be found in Supplementary Material.
For a fair comparison, we also implement SpIN and NeuralEF with a unified I/O interface.
\else
We have open-sourced a PyTorch implementation of our method, along with our implementations of SpIN and NeuralEF with a unified I/O interface for a fair comparison.\footnote{\url{https://github.com/jongharyu/neural-svd}}
\fi

We emphasize that, to apply NeuralSVD (and other existing methods), we only need to know how to evaluate a quadratic form $\langle f|\Top g\rangle$ and inner products such as $\langle f |f'\rangle$. 
Since we consider $\Lr^2$ spaces for most applications, and the quadratic forms and inner products can be estimated via importance sampling or given data in an unbiased manner; see a detailed discussion on importance sampling to \S\ref{app:sec:importance_sampling}.
After all, the gradients described above can be estimated without bias, and we can thus use any off-the-shelf stochastic optimization method with minibatch to solve the optimization problem.
Given a minibatch of size $B$, we can compute the minibatch objective and gradient only with matrix-vector products, and the complexity is $O(B^2L + BL^2)$.

\begin{figure}[t]
    \centering
    \begin{minipage}[t]{.235\textwidth}
        \vspace{0pt}
        \centering
        \includegraphics[width=\textwidth]{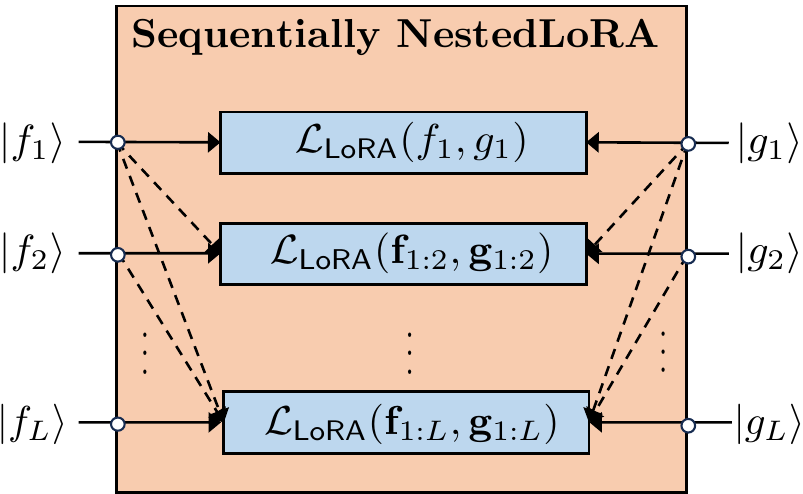}
        \vspace{.6em}
        \subcaption{Sequential nesting.}
    \end{minipage}
    \begin{minipage}[t]{.235\textwidth}
        \vspace{0pt}
        \centering
        \includegraphics[width=\textwidth]{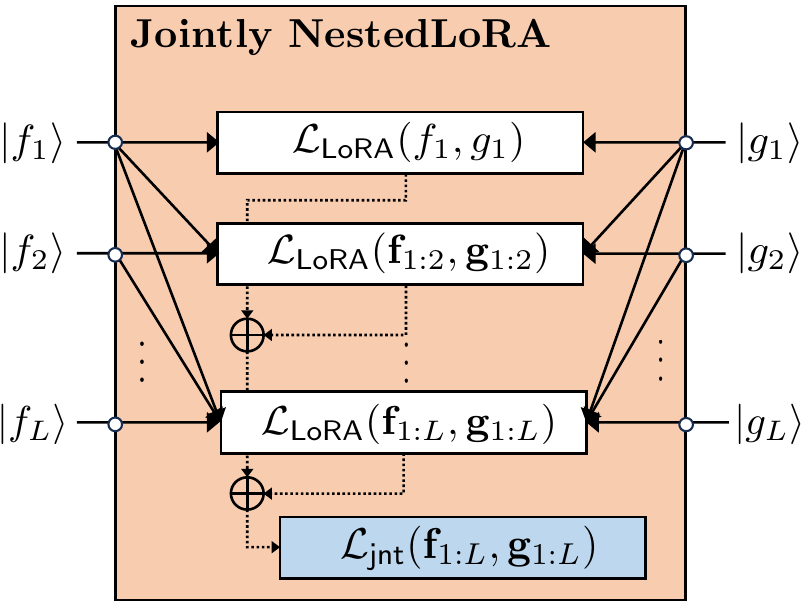}
        \subcaption{Joint nesting.}
    \end{minipage}%
    \caption{Schematic illustrations of the nesting techniques; recall Fig.~\ref{fig:schematic}. The operator and data (and weights for the joint nesting) are omitted for simplicity. 
    In both cases, gradients are computed and backpropagated from the blue boxes. 
    Gradients cannot be backpropagated through the dashed lines in (a); see \S\ref{sec:nesting}.}
    \label{fig:schematic_nesting}
\end{figure}\vspace{-0.5em}

\section{Example Applications and Experiments}
\label{sec:exp}
In this section, we illustrate two example use cases and present experimental results: \emph{differential operators} in computational physics, and \emph{canonical dependence kernels} in machine learning, which will be defined in \S\ref{sec:exp:sketchy}.
We experimentally demonstrate the correctness of NeuralSVD and its ability to learn ordered eigen- or singular-functions and show the superior performance of our method compared to the existing methods. 
We focus on rather small-scale problems that suffice with simple multi-layer perceptrons (MLPs) for extensive numerical evaluation of our method against the existing parametric methods. 
All the training details can be found in \S\ref{app:sec:exp}.

\subsection{Analytical Operators}
\label{sec:exp:pde}
In many application scenarios, an operator $\Top$ is given in an analytical form.
In machine learning, there exists a variety of \emph{kernel-based methods}, which assumes a certain kernel function $k(x,y)$ defined in a closed form. In this case, the underlying operator is the so-called \emph{integral kernel operator} $\Kop$, which is defined as $(\Kop f)(y)\defeq \E_{p(x)}[k(x,y)f(x)]$.
In computational physics, a certain class of important PDEs can be reduced to eigenvalue problems, where we can directly apply our framework to solve them.
In this case, the operator involves a differential operator, such as Laplacian $\nabla^2$, as will become clear below.
We will provide a numerical demonstration of NeuralSVD for the latter scenario.

A representative example of such PDE is a time-independent Schr\"odinger equation (TISE)~\citep{Griffiths--Schroeter2018} 
\[\Hop|\psi\rangle=E|\psi\rangle.\]
Here, $\Hop$ is the Hamiltonian that characterizes a given physical system, $|\psi\rangle$ denotes an eigenfunction, and $E\in\Real$ the corresponding eigen-energy.
Recall that to perform EVD in our SVD framework, we only need to identify $\gv$ to $\fv$.
Since bottom modes are typically of physical interest, we can aim to find the eigenfunctions of the \emph{negative} Hamiltonian $-\Hop$.

We consider two simple yet representative examples of TISEs that have closed-form solutions for extensive quantitative evaluations.
The first example is a 2D hydrogen atom, the corresponding operator of which is compact.
With the second example of a 2D harmonic oscillator, in which the operator of interest is \emph{not} compact, 
we demonstrate that our framework is still applicable.
In both cases, we used simple MLPs with multi-scale random Fourier features as the parametric eigenfunctions~\citep{Wu--Wang--Perdikaris2023}.

\vspace{.5em}\noindent\textbf
{Experiment 1: 2D Hydrogen Atom.}
We first consider a hydrogen atom confined over a 2D plane. By solving the associated TISE, we aim to learn a few bottom eigenstates and their respective eigenenergies.
The detailed problem setting, including the underlying PDE, can be found in \S\ref{app:sec:exp}.
Ignoring irrelevant constants, the true eigenvalues (after negating the sign) are known as $\lambda_{n,l}=(2n+1)^{-2}$ for $n\ge 0$ and $-n\le l\le n$. 
That is, for each $n$, there exist $2n+1$ degenerate states.
In our experiment, we aimed to learn $L=16$ eigenstates that cover the first four degenerate eigen-subspaces. We trained SpIN, NeuralEF, \nsvdseq{}, and \nsvdjnt{} with the same architecture and training procedure with different batch sizes 128 and 512.\footnote{As an exception, a smaller network and batch size 128 was used for SpIN due to its large memory requirement $O(L \times\text{(model size)})$ for maintaining copies of Jacobian for each mode.}
Here, we found that the original version of NeuralEF performed much worse than NeuralSVD,
and we thus implemented and reported the result of a variant of NeuralEF with unbiased gradient estimates, whose definition can be found in
\S\ref{app:sec:neuralef}. 

\begin{figure*}[tb]
\centering
\includegraphics[width=\textwidth]{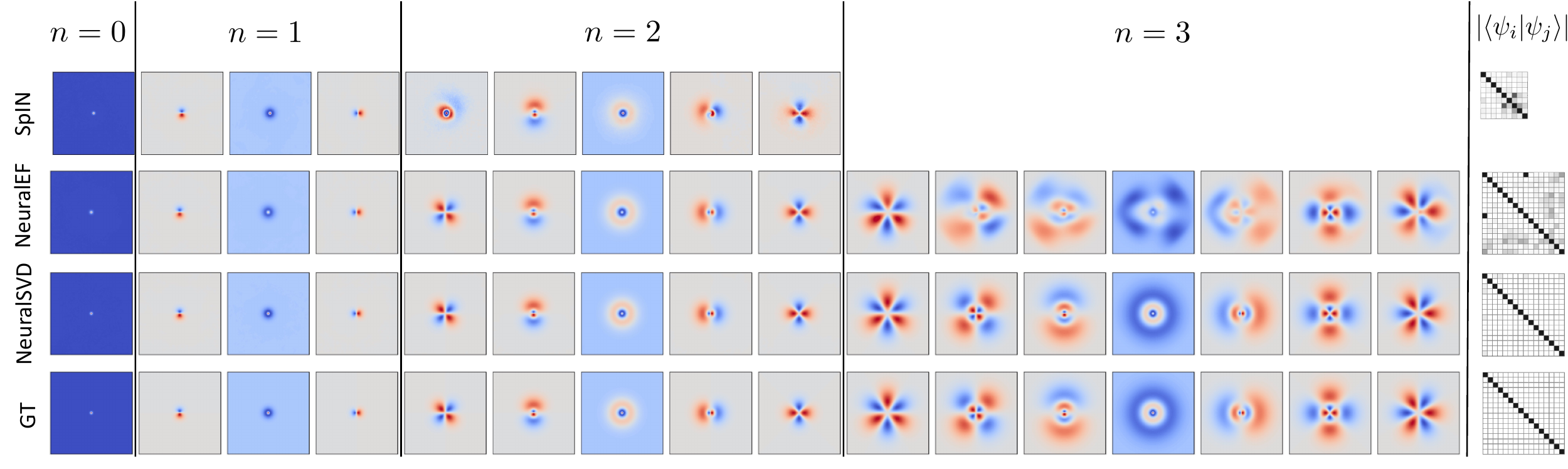}
\caption{Visualization of the first 16 eigenfunctions of the 2D hydrogen atom. 
The first three rows present the learned eigenfunctions by SpIN (128), NeuralEF (512), and \nsvdseq{} (512), respectively.
Due to the memory complexity of SpIN, we ran SpIN with only 9 eigenstates.
The learned functions are aligned by an orthogonal transformation via the orthogonal Procrustes method within each degenerate subspace to compare with the ground truth (GT) in the fourth row. 
The rightmost column visualizes the learned orthogonality.
}
\label{fig:hydrogen}
\end{figure*}

Fig.~\ref{fig:hydrogen} shows the learned eigenfunctions from SpIN (128), NeuralEF (512), and \nsvdseq{} (512), where the numbers in the parentheses indicate used batch sizes.
For comparison, we present the true eigenfunctions with a choice of canonical directions to plot the degenerate subspaces (last row).
Note that SpIN and NeuralEF do not match the ground truths even after the rotation in several modes.
Further, the learned functions (before rotation) are not orthogonal as visualized in the rightmost column.
In contrast, \nsvdseq{} can reliably match the correct eigenfunctions, with almost perfect orthogonality.
Fig.~\ref{fig:pde_summary_bar_full} report several quantitative measures to evaluate the fidelity of learned eigenfunctions; see \S\ref{app:sec:def_measures} for the definitions of the measures.
The results show that both \nsvdjnt{} and \nsvdseq{} outperform SpIN and NeuralEF by an order of magnitude.

Note that though \nsvdjnt{} can recover the eigenfunctions reasonably well, it performs worse than \nsvdseq{} as expected; see Remark~\ref{rem:which_version_nesting}.
We also remark that the computational and memory complexity of NeuralEF and NeuralSVD are almost the same, while SpIN takes much longer time and consumes more memory due to the Cholesky decomposition and the need for storing the Jacobian; we refer an interested reader to \S\ref{app:sec:spin} for the detail of SpIN.

In the Appendix, we demonstrate the advantages of NeuralSVD compared to standard numerical linear algebra techniques; see
\S\ref{app:sec:lobpcg} for its comparison to a matrix-free method, and \S\ref{app:sec:rayleigh_ritz} for the effectiveness of nesting.

\vspace{.5em}\noindent\textbf
{Experiment 2: 2D Harmonic Oscillator.}
We now consider finding the eigenstates of a 2D harmonic oscillator,
whose eigenstate is characterized by a pair of nonegative integers $(n,l)$ for $n\ge 0$ and $0\le l \le n$ with (negative) eigenenergy $\lambda_{n,l}=-2(n+1)$ and multiplicity of $n+1$.
In contrast to the 2D hydrogen case, it is clear that the negative Hamiltonian is neither PD nor compact.
To retrieve eigenfunctions even in this case, we can consider a \emph{shifted} operator $\Top+c\Ic$, where $\Ic$ is an identity operator and $c\ge 0$ is a constant, so that the spectrum becomes $\lambda_{n,l}=c-2(n+1)$.
Note that shifting only affects the quadratic form $\langle f | \Top+c\Ic|f\rangle = \langle f|\Top f\rangle + c\|f\|^2$.

As an example, we chose $c=16$, so that $\lambda_{n,l}>0$ for $0\le n\le 6$.
We claim that NeuralSVD recovers the eigenfunctions with positive eigenvalues, the first $28(=1+\ldots+7)$ states for this case, and the nonpositive part will converge to the constant zero function; see Theorem~\ref{thm:schmidt_evd_noncompact} in \S\ref{app:sec:noncompact}.
We note that the LoRA objective~\eqref{eq:lora_obj} is still well-defined even when $\Top$ is not compact.
While other methods are also applicable and can recover the positive part in principle, the learned functions will be arbitrary for the nonpositive part, unlike NeuralSVD learning zero functions.
This implies that one can correctly infer the nonpositive part by computing the norms of the NeuralSVD eigenfunctions.

We report the quantitative measures in 
Fig.~\ref{fig:pde_summary_bar_full}(b), where only the positive part, \ie the first 28 eigenstates, 
was taken into account for the evaluation.
Note that SeqNestedLoRA significantly outperforms NeuralEF in this example as well.
Moreover, as explained above, the norms of the learned eigenfunctions with NeuralSVD well approximate the ground truth eigenvalues for the positive part, and almost zero for the non-positive part (data not shown); see \S\ref{app:sec:norms} for the spectrum estimation with NeuralSVD based on function norms.
In contrast, one cannot distinguish whether learned eigenfunctions are meaningful or not only based on the learned eigenvalues from NeuralEF.

\begin{figure*}
    \centering
    \includegraphics[width=.9\textwidth]{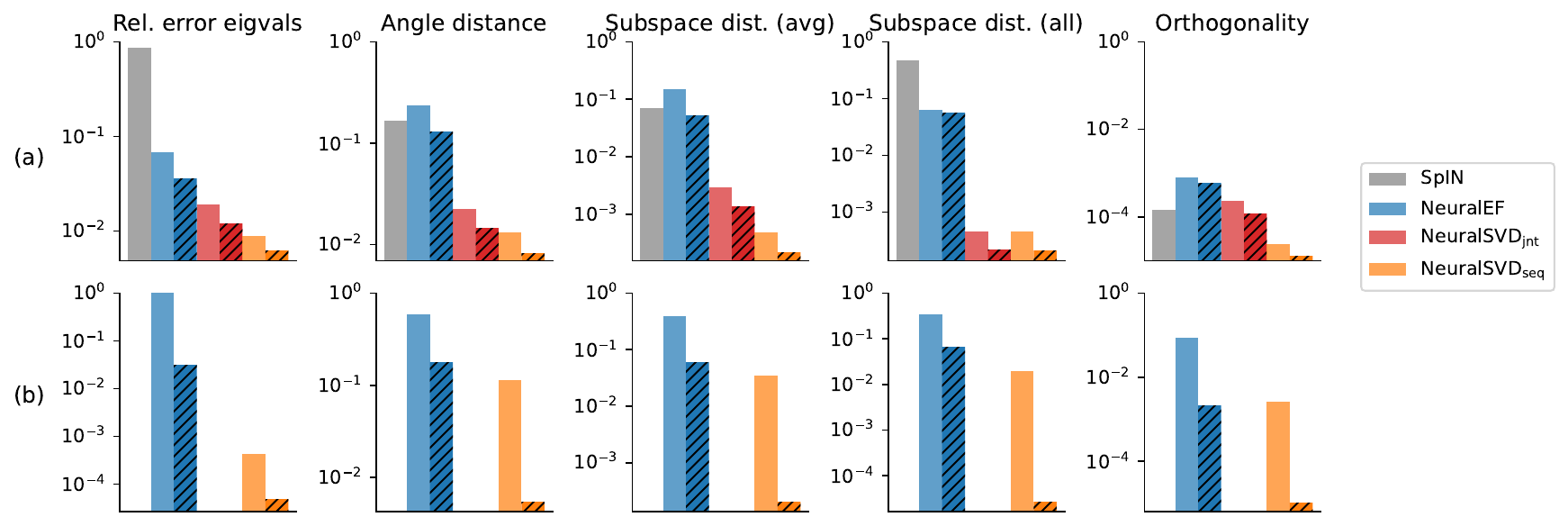}
    \caption{Summary of quantitative evaluations for solving TISEs: (a) 2D hydrogen atom; (b) 2D harmonic oscillator.
    Non-hatched, light-colored bars represent a batch size of 128, and hatched bars indicate 512.
    The definitions of reported measures are given in \S\ref{app:sec:def_measures}.
    }
    \label{fig:pde_summary_bar_full}
\end{figure*}

\subsection{Data-Dependent Operators}
\label{sec:exp:sketchy}
Beyond analytical operators, we can also consider a special type of \emph{data-dependent} kernels.
Given a joint distribution $p(x,y)$, we define 
\[k(x,y)\defeq \frac{p(x,y)}{p(x)p(y)}\] which we call the \emph{canonical dependence kernel} (CDK).
Although the CDK cannot be explicitly evaluated, it naturally defines the similarity between $x$ and $y$ based on their joint distribution, and thus can better capture the statistical relationship than a fixed, analytical kernel.
Note that its induced integral kernel operator is the conditional expectation operator, \ie $(\Kop g)(x) = \E[g(Y)|X=x]$ and $(\Kop^* f)(y) = \E[f(X)|Y=y]$, where $\Kop^*$ denotes the \emph{adjoint} of $\Kop$.

CDK appears and plays a central role in several statistics and machine learning applications, and various connections of CDK to the existing literature such as Hirschfeld--Gebelein--R\'{e}nyi (HGR) maximal correlation~\citep{Hirschfeld1935,Gebelein1941,Renyi1959}
are discussed in \S\ref{app:sec:related_work}.

One special property of CDK is that we can compute the objective function using paired samples, even though we do not know the kernel value ${p(x,y)}/{(p(x)p(y))}$ in general.
That is, the ``operator term'' $\langle g_\ell | \Kop f_\ell \rangle$ can be computed as
\begin{align}
\label{eq:oneshot_formula_max_corr}
\langle g_\ell | \Kop f_\ell \rangle
= \E_{p(x,y)}[f_\ell(x)g_\ell(y)],
\end{align}
where we change the measure $p(x)p(y)$ with $p(x,y)$ by the definition of $k(x,y)$.
Since the first singular functions are always constant functions, we can simply replace $\fv(x)^\intercal\gv(y)$ with $1+\fv(x)^\intercal\gv(y)$ to exclude the trivial mode,
so that we can recover the second singular functions and on. 
\ifforreview
We note in passing that it is equivalent to decomposing kernel $\frac{p(x,y)}{p(x)p(y)}-1$.
\else
We note in passing that it is equivalent to decomposing kernel $\frac{p(x,y)}{p(x)p(y)}-1$,
which is the convention used in a line of literature; see, \eg \citep{Huang--Makur--Wornell--Zheng2019,Xu--Zheng2023}.
\fi

\begin{table*}[t]
    \centering
    \caption{Evaluation of the ZS-SBIR task with the Sketchy Extended dataset~\citep{Sketchy2016,SketchyExtended2017}.
    We note that the two baselines require ($\bast$) a generative model, while NeuralSVD can learn representations directly without such.}
    \begin{small}
    \begin{tabular}{r c c c c c c} 
    \toprule    
        \textbf{Method} 
        & \textbf{Ext. knowledge}
        & \textbf{Gen. model}
        & \textbf{Structured}
        &  \textbf{P@100} & \textbf{mAP} & \textbf{Split}  \\
     \midrule
        LCALE~\citep{Lin--Xu--Gao--Wang--Shen2020} 
        & Word embeddings 
        & $\bast$
        & \xmark
        & 0.583 & 0.476 & 1\\
        IIAE~\citep{Hwang--Kim--Hong--Kim2020IIAE} 
        &  
        & $\bast$
        & \xmark
        & 0.659 & 0.573 & 1\\
    \midrule
        \multirow{2}{*}{
        \nsvdjnt{}
        } 
        & \multirow{2}{*}{}
        & \multirow{2}{*}{}
        & \multirow{2}{*}{\cmark}
        & {${\bf 0.670}_{\pm0.010}$} & ${\bf 0.581}_{\pm0.008}$ & 1\\ 
        & \multirow{2}{*}{}
        & \multirow{2}{*}{}
        & & {${\bf 0.724}_{\pm0.008}$}  & 
        {${\bf 0.641}_{\pm0.008}$} & 2\\
    \bottomrule
    \end{tabular}
    \end{small}
    \label{tab:sketchy}
    \vspace{-0.2em}
\end{table*}

\vspace{0.5em}\noindent\textbf{Application: Cross-Domain Retrieval.}
One natural application of the CDK is in the cross-domain retrieval problem. 
Specifically, here we consider the \emph{zero-shot sketch-based image retrieval} (ZS-SBIR) task proposed by \citet{Yelamarthi--Reddy--Mishra--Mittal2018}. The goal is to construct a good model that retrieves relevant photos $y_i$'s from a given query sketch $x$, only using a training set with no overlapping classes in the test set (hence called \emph{zero-shot}).

To obtain coembeddings of sketches and photos from the CDK framework, we define a natural joint distribution $p(x,y)$ for sketch $x$ and photo $y$, by picking a random pair of $(x,y)$ from the same class. 
Formally, the joint distribution is 
defined as $p(x,y)=\E_{p(c)}[p(x|c)p(y|c)]$, where $p(c)$ denotes the class distribution, and $p(x|c)$ and $p(y|c)$ the class-conditional sketch and photo distributions, respectively.
We emphasize that the resulting joint distribution is \emph{asymmetric}, since $x$ and $y$ are two different modalities, and thus the existing frameworks, such as SpIN or NeuralEF cannot be directly applied.
We also note that the matrix approach, which computes the empirical CDK matrix and then performs SVD, is infeasible, as density ratio estimation for constructing the kernel matrix is nontrivial in the high-dimensional space.
In sharp contrast, we can learn to decompose the CDK $k(x,y)=\frac{p(x,y)}{p(x)p(y)} \approx 1 + \fv(x)^\intercal \gv(y)$ directly with NeuralSVD. 

After learning the functions $\fv$ and $\gv$, for a given query $x$, we can retrieve based on the highest inner-product $\fv(x)^\intercal\gv(y)$ from $y\in\{y_1,\ldots,y_N\}$. 
This approach has a natural probabilistic interpretation: ``retrieve $y$, if $y$ is more likely to appear together than independently, \ie $p(x,y)\gg p(x)p(y)$''.
In addition to the interpretable retrieval scheme, the retrieval system can benefit from the learned spectral structure. That is, when successfully learned, 
the NeuralSVD representations are ideally stacks of \emph{ordered} top-$L$ singular functions of the CDK.
The representations can thus be called \emph{structured} in the sense that the coordinates of representations are ordered by the associated singular values, and also each coordinate encodes exclusive information since different coordinates are constructed so as to be effectively orthogonal.
We can thus potentially reduce the dimensionality of the embedding, by keeping only informative coordinates.

\vspace{0.5em}\noindent\textbf{Experiment.}
We aimed to learn $L=512$ singular functions, parameterizing them by a single network. 
We followed the standard training setup in the literature~\citep{Hwang--Kim--Hong--Kim2020IIAE}.
We report the Precision@100 (P@100) and mean average precision (mAP) scores on the two test splits in the literature; we defer the definition of these metrics to \S\ref{app:sec:cdk_exp_detail}.
We empirically found that \nsvdseq{} performed much worse than \nsvdjnt{} as discussed in Remark~\ref{rem:which_version_nesting}, only achieving Precision@100 around 0.2. Hence, we only report the result from \nsvdjnt{}.
Since SpIN and NeuralEF are not directly applicable to asymmetric kernels, we do not include them in the comparison.

Table~\ref{tab:sketchy} summarizes the evaluation. It shows that the CDK-based retrieval learned by NeuralSVD, albeit simple, can outperform the state-of-the-art representation learning methods based on generative models, including a method that incorporates additional knowledge.

\begin{figure}[t]
\centering
\includegraphics[width=.4\textwidth]
{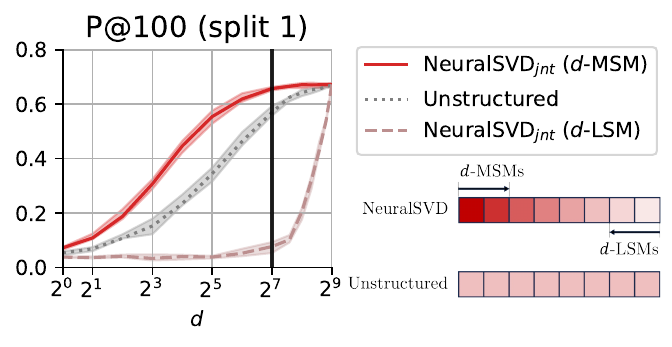}
\vspace{-.5em}
\caption{P@100 performance of NeuralSVD on ZS-SBIR task, with varying dimensions.
NeuralSVD can achieve full performance only with one-quarter ($2^7=128$) of the full dimensions, sweeping from the most significant modes.}
\label{fig:sketchy}\vspace{-1em}
\end{figure}

\bluet{As alluded to above, moreover, we can demonstrate that NeuralSVD learns \emph{structured} representations, while the baselines only learn ``unstructured'' ones.}
To illustrate, we verify that most information relevant for retrieval is concentrated around the top modes.
To illustrate this, we repeated the following evaluation with 10 random initializations and report the summary in Fig.~\ref{fig:sketchy}.
First, from the NeuralSVD representations of $2^9=512$ dimensions, we used the $d$ most significant modes ($d$-MSMs) for $d\in\{1,2,2^2,\ldots, 2^9\}$, and evaluated the retrieval performance based on the similarity measure $\fv_{[d]}(x)^\intercal\gv_{[d]}(y)$. 
The performance rapidly grows as the dimensionality $d$ gets large, and the best is almost achieved at about $2^7=128$, which is only a quarter of the full dimension; see the 
red
lines (NeuralSVD) and the vertical 
black
lines.
Also, we can empirically validate that the learned representations are almost perfectly orthogonal.

As a further investigation, we consider two additional baselines.
First, from the NeuralSVD representations, we evaluated the retrieval performance with the $d$ least significant modes ($d$-LSMs) for $d\in\{1,2,2^2,\ldots, 2^9\}$; see pink, dashed lines labeled with ``NeuralSVD ($d$-LSM)''.
The retrieval performance is very poor even when using the bottom 128 dimensions, which indicates that the LSMs do not encode much information.
Second, we trained an \emph{unstructured} embedding by training the same network with the LoRA objective without nesting, so that the network only learns the top-512 subspace of CDK; see gray, dotted lines labeled with ``Unstructured''.
As expected, its retrieval performance lies in the middle of NeuralSVD (MSM) and NeuralSVD (LSM). 
Hence, we can conclude that the learned representations with NeuralSVD are well-structured and effectively encode the information in a compact manner.

\begin{remark}[Impact of imperfect orthogonality]
Since our NeuralSVD framework only implicitly promotes the orthogonality without hard constraints, the learned singular-functions may only exhibit orthogonality to each other in an approximate manner, and
such deviations from orthogonality may impact the performance in a downstream task. In the PDE example, if the goal is to find accurate eigenvalues (i.e., eigenenergies of the system), then slightly imperfect orthogonality across non-degenerate modes may result in slightly less accurate eigenvalue estimates. In the representation learning example, if the retrieval performance is the only criterion for the quality of representations, slightly imperfect orthogonality would result in “slightly less structured” representations in that different coordinates might share some redundant information, which will impact the ``compressibility'' of the representations. All in all, imperfect orthogonality could affect different tasks differently, but we provide an empirical showcase that almost perfect orthogonality can be guaranteed in the present examples.
\end{remark}

\section{Concluding Remarks}
\label{sec:discussion}
In this paper, we proposed NeuralSVD, a new optimization framework for learning parametric singular- or eigen-functions of a linear operator via NestedLoRA. 
Given the efficient unconstrained optimization framework, practitioners can focus on selecting the most suitable parametric functions (or good architectures) and optimization algorithms to meet the practical requirements of their specific problems.
We could potentially extend the applicability of the existing classical algorithms based on SVD/EVD in various fields, \eg quantum chemistry~\citep{Hermann--Schatzle--Noe2020,Pfau--Spencer--Matthews2020} or spectral embedding methods~\citep{Shi--Malik2000,Belkin--Niyogi2003} for large-scale, high-dimensional data, combined with the use of powerful neural networks.

\vspace{.5em}\noindent\textbf{Limitations and Future Directions.}
We conclude with two important limitations and future directions to further advance the applicability of the parametric approach.
\begin{itemize}
\item First, the parametric approach is less explored than the nonparametric approach. The challenge is to understand when a large network can approximate a given operator and to determine an optimization algorithm that converges to the desired global optimizer, such as NestedLoRA. Addressing this gap and providing performance guarantees is a valuable research direction.
\item Second, users of the parametric approach must choose an appropriate function and optimization hyperparameters. Our investigation has shown the effectiveness of simple MLP architectures and specific hyperparameters in our examples. However, for larger applications, scalability challenges may require more sophisticated architectures and fine-tuning. We advocate for future research to design effective network architectures tailored to specific operators and tasks.
\end{itemize}

%% file: appendix.tex
\section{On Standard Linear Algebra Techniques}
\subsection{Empirical SVD and Nystr\"om Method}
\label{app:sec:nystrom}
A standard variational characterization of SVD is based on the following sequence of optimization problems:
\begin{align}\label{eq:svd_char}
\begin{aligned}
\underset{\phit_\ell\in\Fc,\psit_\ell\in\Gc}{\text{maximize~~}} & 
\langle \psit_\ell | \Top \phit_\ell\rangle,\\
\text{subject~to~~} &
\langle \phit_i | \phit_\ell\rangle 
=\langle \psit_i | \psit_\ell\rangle
=\d_{i\ell} \quad\forall i\in [\ell]\bluet{\defeq \{1,\ldots,\ell\}}.
\end{aligned}
\end{align}
If $\Top$ is compact and the previous $(\ell-1)$ pairs of functions $\{(\phit_i,\psit_i)\}_{i\in [\ell-1]}$ are the top-$(\ell-1)$ singular functions, then
the maximum value of the $\ell$-th problem, is attained by the $\ell$-th singular functions $(\phi_\ell,\psi_\ell)$~\citep[Proposition~A.2.8]{Bolla2013}.

While the notion of SVD and its variational characterization are mathematically well defined, we cannot solve the infinite-dimensional problem \eqref{eq:svd_char} directly in general, except a very few cases with known closed-form solutions.
Hence, in practice, a common approach is to perform the SVD of an \emph{empirical kernel matrix} induced by finite points (samples, in learning scenarios). 
That is, given $x_1,\ldots,x_M\sim p(x)$ and $y_1,\ldots,y_N\sim p(y)$, we define the empirical kernel matrix $\hat{\Km}\in \Real^{M\times N}$ as $(\hat{\Km})_{ij}\defeq k(x_i,y_j)$.
Suppose we perform the (matrix) SVD of $\hat{\Km}/\sqrt{MN}$ and obtain the top-$L$ left- and right-singular vectors 
$\hat{\Um}=[\hat{\uv}_1,\ldots,\hat{\uv}_L] \in \Real^{M\times L}$ and
$\hat{\Vm}=[\hat{\vv}_1,\ldots,\hat{\vv}_L]\in \Real^{N\times L}$ (normalized as $\hat{\Um}^\intercal \hat{\Um}=M\Imatrix$ and $\hat{\Vm}^\intercal \hat{\Vm}=N\Imatrix$, \bluet{where $\Imatrix$ denotes the identity matrix}) with the top-$L$ singular values $\hat{\sigma}_1\ge \ldots\ge \hat{\sigma}_L\ge 0$.
Then, for each $\ell$, $\hat{\uv}_\ell$ and $\hat{\vv}_\ell$ approximate the evaluation of $\phi_\ell$ and $\psi_\ell$ at training data, \ie
\[
\hat{\uv}_\ell \approx [\phi_\ell(x_1),\ldots,\phi_\ell(x_M)]^\intercal, \quad
\hat{\vv}_\ell \approx [\psi_\ell(y_1),\ldots,\psi_\ell(y_N)]^\intercal,
\quad \text{and}\quad\hat{\sigma}_\ell \approx \sigma_\ell.
\]
Hence, for $\ell\ge 1$ with $\hat{\sigma}_\ell>0$, the $\ell$-th left-singular function at $x$ can be estimated as
\begin{align}
\label{eq:nystrom}
\hat{\phi}_{\ell}(x)
\defeq 
\frac{\hat{\sigma}_\ell^{-1}}{N}\sum_{j=1}^N 
k(x,y_j) (\hat{\vv}_\ell)_{j},
\end{align}
which is a finite-sample approximation of the relation $\phi_\ell(x)={\sigma_\ell^{-1}} (K \psi_\ell)(x)$.
This is often referred to as the \emph{Nystr\"om method}; see, \eg \citep{Williams--Seeger2000,Bengio--Delalleau--Le-Roux--Paiement--Vincent--Ouimet2004}.

Performing SVD of the kernel matrix $\Km$ can be viewed as solving \eqref{eq:svd_char} with finite samples in the nonparametric limit.
The sample SVD approach is limited, however, due to its memory and computational complexity.
The time complexity of full SVD is $O(\min\{MN^2,M^2N\})$ not scalable, but there exist iterative subspace methods that can perform top-$L$ SVD in an efficient way.
Note, however, that the data matrix should be stored in memory to run standard SVD algorithms, which may not be feasible for large-scale data.
Moreover, while the query complexity $O(N)$ or $O(M)$ of the Nystr\"om method could be reduced by choosing a subset of training data, the challenge posed by the curse of dimensionality can potentially undermine the reliability of the Monte Carlo approximation~\eqref{eq:nystrom} as an estimator.

\subsection{On the Effectiveness of Nesting vs. the Rayleigh--Ritz Method}
\label{app:sec:rayleigh_ritz}
One may question the advantages of learning the ordered eigenfunctions via nesting compared to first learning the subspace and then determining the order within the subspace using the \emph{Rayleigh--Ritz method} in numerical linear algebra, which is a numerical algorithm to approximate eigenvalues~\citep{Trefethen--Bau2022}. 
The idea of Rayleigh--Ritz is to use an orthonormal basis of some smaller-dimensional subspace and solve the surrogate eigenvalue problem of smaller dimension projected on the subspace.
The quality of the Rayleigh--Ritz approximation depends on the user-defined orthonormal basis $\{|\phi_j\rangle\}_{j=1}^d$. 
That is, as the basis better captures the desired eigenmodes of the target operator, the approximation becomes more accurate.
Otherwise, for example, if an eigenmode is orthogonal to the subspace of the eigenbasis, it cannot be found by this procedure.
For completeness, we describe the procedure at the end of this section.

Given this standard tool, one may ask whether it is necessary to learn the ordered singular functions as done by NeuralSVD, SpIN, and NeuralEF.
Instead, since by minimizing LoRA objective we can approximately learn the top-$L$ eigensubspace (Theorem~\ref{thm:subspace}), 
one can consider applying Rayleigh--Ritz with the learned functions trained by LoRA.
Though the idea is valid and the full EVD of $L\times L$ matrix in Rayleigh--Ritz would be virtually at no additional cost, we remark that the two-stage procedure has several drawbacks compared to the direct approach with NeuralSVD.
First, note that the learned functions with the LoRA objective are necessarily orthogonal and Gram--Schmidt process should be applied for obtaining the orthonormal basis before Rayleigh--Ritz.
Note, however, that Gram--Schmidt becomes nontrivial in function spaces, as we need to compute the inner products and norms of functions at each step. 
Moreover, computing the inner products $\langle \phi_i|\Tc \phi_j\rangle$ to compute the reduced operator as described below may introduce an additional estimation error.

More crucially, we empirically verified that \nsvdseq{} or \nsvdjnt{} can attain lower subspace distance than that learned by LoRA (without nesting), while being able to correctly ordered orthogonal eigenbasis simultaneously. Since the quality of Rayleigh--Ritz is limited by the quality of the given subspace, if the learned subspace has lower quality, the outcome must be worse.
For example, in the 2D hydrogen atom experiment, we observed that the subspace distance over the 16 eigenmodes was $3.56\times 10^{-4}${\scriptsize${\pm 6.60\times 10^{-5}}$} with LoRA, while $2.12\times 10^{-4}${\scriptsize${\pm 2.09\times 10^{-5}}$} and $2.06\times 10^{-4}${\scriptsize${\pm 1.91\times 10^{-5}}$} were attained by \nsvdjnt{} and \nsvdseq{}, respectively; see Fig.~\ref{fig:pde_summary_bar_full}(a).
Since nesting does not increase complexity compared to the non-nested case via its efficient gradient implementation with masking, we argue that NeuralSVD can be more efficient than the two-stage approach.

\paragraph{Rayleigh--Ritz for Operator EVD.}
For the ease of exposition, here we describe the procedure for an operator eigenvalue problem.
Given a self-adjoint operator $\Top$, suppose that we wish to solve an eigenvalue problem
\[
\Top|\psi\rangle = \lambda|\psi\rangle.
\]
Since the problem may be hard to solve directly, the Rayleigh--Ritz method assumes that a set of orthonormal functions $\{|\phi_1\rangle,\ldots,|\phi_d\rangle\}$ for some $d\ge 1$, preferably $d\ll N$, and define $\Bm\in\Real^{d\times d}$ such that $\Bm_{ij} \defeq \langle \phi_i|\Tc\phi_j\rangle$.
Then, we solve the eigenvalue problem 
\[
\Bm\yv = \mu\yv.
\]
Given an eigenpair $(\mu_i,\yv_i)$, we compute the Ritz function $|\tilde{\psi}_i\rangle\defeq \sum_{j=1}^d y_j |\phi_j\rangle$,
and set the Ritz value $\tilde{\lambda}_i\defeq \mu_i$.
The output of the Rayleigh--Ritz method are the Ritz pairs $\{(\tilde{\lambda}_i,\tilde{\psi}_i)\}_{i=1}^d$.

\subsection{Comparison to a Nonparametric Approach}
\label{app:sec:lobpcg}
A reader familiar with numerical linear algebra literature may wonder how the parametric approach is compared to the standard techniques.
To this end, as a quick comparison, we performed the following baseline experiment, with one of the standard matrix-free techniques called
``Locally Optimal Block Preconditioned Conjugate Gradient'' (LOBPCG) for finding top-$L$ eigenvalues of large matrices~\citep{Knyazev2001,Knyazev2017}.
To learn the first $L=16$ eigenstates of the 2D hydrogen atom, we consider a truncated domain $[-50,50]$ and discretize each axis by $N$ grid points. We then perform the top-$L$ EVD of the discretized Hamiltonian matrix of size $N\times N$ using LOBPCG (using the PyTorch functionality \texttt{torch.lobpcg}). The result is summarized in Fig.~\ref{fig:lobpcg}. In the first panel, we present the relative errors in the estimated eigenvalues in parallel to Fig.~\ref{fig:pde_summary_bar_full}(a). 
In the second panel, the blue line summarizes the average absolute relative error for each $N$. 
The accuracy improves as $N$ becomes larger as expected in general, but we observe that the quality of estimates of latter eigenvalues become worse with $N=1600$ than with $N=800$.
Compared to the best result obtained by \nsvdseq{}(512) (indicated by the red dashed horizontal line), this naive approach may take substantially more time to achieve comparable accuracy as it may not scale well as $N$ increases as shown in the third panel. This briefly showcases a possible advantage of NeuralSVD (or the parametric approach at large) over the matrix-based approach.

\begin{figure}[ht]
\centering
\includegraphics[width=.6\textwidth]{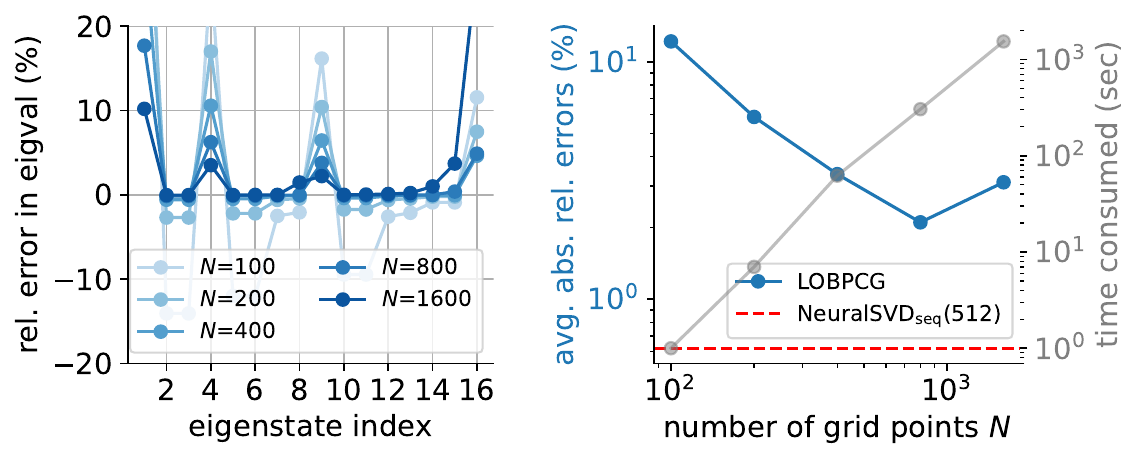}
\caption{Performance of LOBPCG 2D hydrogen experiment.}
\label{fig:lobpcg}
\end{figure}

We remark, however, some caveats in this comparison. 
First, the LOBPCG implementation we used here might not be fully optimized and there could exist a version that exhibits better scalability. The runtime could be also drastically reduced by using GPU or parallel machines as such numerical linear algebra algorithms are known to be very well optimized for such resources, while the current experiment was run on a CPU machine. Second, the discretization is rather naive and a more sophisticated discretization with some choice of orthonormal basis could lead to a better solution. Third, the relatively large error in the first eigenvalue estimate is seemingly due to the non-differentiable cusp of the first eigenfunction, and thus the discretization approach could behave better for other examples.

\section{Related Work}
\label{app:sec:related_work}
\subsection{General Literature Review}
\subsubsection{Low-Rank Approximation}
The theory of low-rank approximation was initially developed to solve partial differential equations (PDEs), including \citet{Schmidt1907}'s work; see, \eg \citep{Stewart2011}.
A special case for finite-dimensional matrices was independently discovered later by  \citet{Eckart--Young1936} and \citet{Mirsky1960}, which are perhaps better known in the literature. We refer an interested reader to \citep{Stewart1993} for detailed historical remarks.
For matrices, \citet{Mirsky1960} extended the low-rank approximation theory to any unitarily invariant norms.
While it would be interesting to extend the proposed framework in the current paper for other norms, they do not seem to easily admit an optimizable objective function.

\subsubsection{Canonical Dependence Kernels}
Interestingly, there exists a rich literature on decomposing canonical dependence kernels (CDK); see \S\ref{sec:exp:sketchy}.
The CDK has a close relationship to the Hirschfeld--Gebelein--R\'{e}nyi (HGR) maximal correlation~\citep{Hirschfeld1935,Gebelein1941,Renyi1959}.
Note that the first singular functions are trivially constant functions, and the corresponding singular value $\sigma_1$ is always 1.
When $L=2$, the second singular value is known as the HGR maximal correlation. 
In general, the optimization problem can be understood as the high-dimensional extension of the maximal correlation; for given $L\ge 2$, the optimal functions $\phib^\star$ and $\psib^\star$ are the optimal $L$-dimensional projections of $x\sim p(x)$ and $y\sim p(y)$ that are maximally correlated.
The CDK plays an important role in learning applications and has been frequently redeveloped, bearing different names, \eg correspondence analysis~\citep{Greenacre1984} and \emph{principal inertia components}~\citep{Hsu--Salamatian--Calmon2019,Hsu--Salamatian--Calmon2022} for finite alphabets, the contrastive kernel~\citep{HaoChen--Wei--Gaidon--Ma2021,Deng--Shi--Zhu2022} and the pointwise dependence~\citep{Tsai--Wu--Salakhutdinov--Morency2020} in the self-supervised representation learning setup.

The nonnested objective $\loraobj(\fv,\gv)$ for CDK was proposed and studied in  \citet{Wang--Wu--Huang--Zheng--Xu--Zhang--Huang2019} and related works, \eg \citep{Xu--Huang--Zheng--Wornell2022}. 
The H-score was first introduced by \citet{Wang--Wu--Huang--Zheng--Xu--Zhang--Huang2019} who coined the term H-score (or Soft-HGR), for learning HGR maximal correlation functions with neural networks. 
It also appeared as a local approximation to log-loss of classification deep neural networks~\citep{Xu--Huang--Zheng--Wornell2022}.
We mention in passing that the nonnested objective has been recently proposed independently under the name of the \emph{spectral contrastive loss}~\citep{HaoChen--Wei--Gaidon--Ma2021}, specifically when the CDK is induced by the random augmentation from the standard self-supervised representation learning setup. 
A recent work~\citep{Hu--Principe2022} proposes to learn features of two modalities based on the EVD of the so-called cross density ratio, which can be equivalently understood as the CDK of a symmetrized joint distribution. 
This paper, however, also only aims to characterize the top-$L$ subspace without the structure. Their optimization problem is based on minimizing the log-determinant of a normalized autocorrelation function; compared to our LoRA objective, the resulting optimization inherently suffers from biased gradients, which may lead to issues in practice.

\subsubsection{Nesting}
The idea of \emph{joint} nesting was first introduced by \citet{Xu--Zheng2023} as a general construction to decompose multivariate dependence, which is equivalent to CDK in our terminology, for learning structured features; see the paper for more detailed discussion. 
The joint nesting proposed in this paper can be understood as an extension of the idea to general operators beyond CDK.
While the idea of sequential nesting with LoRA is new, we observe that it conceptually resembles the idea of a class of streaming PCA algorithms such as \citep{Sanger1989,Gemp--McWilliams--Vernade--Graepel2021}, in which the $(\ell+1)$-th eigenvector is updated under the assumption that the estimates for the first $\ell$ eigenvectors are accurate.

We note that a recent work~\citep{Kusupati--Bhatt--Rege--Wallingford--Sinha--Ramanujan--Howard-Snyder--Chen--Kakade--Jain--Farhadi2022} proposed learning a structured representation using a concept similar to the joint nesting technique introduced in our current paper. The method is referred to as \emph{Matryoshka Representation Learning} (MRL). 
The key difference in MRL is that it assumes a labeled image dataset and uses the multi-class softmax cross-entropy loss function as its constituent loss function for ``nesting''. Compared with MRL, the features learned by our CDK-based representation learning framework are interpretable as singular functions of the dependence kernel. 
This fundamental relation provides the learned features with theoretical guarantees, such as the uncorrelatedness of features. It is also worth noting that our features defined by the global minimizer of NestedLoRA$_{\mathsf{jnt}}$ is invariant to the choice of weights, while a different choice of such weights in MRL would characterize different features. Our framework is also not restricted to the supervised case, as demonstrated in the cross-domain retrieval example (\S\ref{sec:exp:sketchy}).

\subsubsection{Other Correlation Analysis Methods}
There exists another line of related literature in correlation analysis.
Deep canonical correlation analysis (DCCA)~\citep{Andrew--Arora--Bilmes--Livescu2013} can be understood as solving a restricted HGR maximal correlation problem, searching over a class of neural networks instead of all measurable functions. 
The DCCA objective function, however, requires a nontrivial optimization technique, 
does not recover the order, and it cannot be easily extended to find a large number of modes.
The correspondence-analysis neural network (CA-NN)~\citep{Hsu--Salamatian--Calmon2019} also aims to decompose the CDK based on a different optimization framework, but they involve the $L$-th Ky-Fan norm and the inversion of $L\times L$ matrix, which complicate the optimization procedure.
Instead of deploying neural networks, \citet{Michaeli--Wang--Livescu2016} proposed to decompose an empirical CDK matrix constructed by the Gaussian kernel density estimators and coined the method as nonparametric canonical correspondence analysis (NCCA).

\subsubsection{Neural-Network-Based Methods for Eigenvalue Problems}
As alluded to earlier, there exists a rather separate line of work on solving eigenvalue problems using neural networks for solving linear PDEs that are in the form of an eigenvalue problem (EVP) in the physics or scientific computing community.
Given the vastness of the literature and the rapid evolution of the field, providing a comprehensive overview is challenging. Nonetheless, we will emphasize key concepts and ideas.

\vspace{.5em}\noindent\textbf{Computational Physics Literature.}
The idea of using neural networks for solving PDEs which can be reduced to eigenvalue problems dates back to \citep{Lagaris--Likas--Fotiadis1997}, where an explicit Gram--Schmidt process was proposed to attain multiple eigenstates.
Unlike the methods in the machine learning literature, many recent works rely on minimizing the sum of residual losses, mostly in the form of $\|(\Top-\lambda_\ell\Iop)\phi_\ell\|$ where $\lambda_\ell$ needs be also optimized or estimated from $|\phi_\ell\rangle$, with regularization terms that penalize the normalization of and the orthogonality between the parametric functions; see 
\citep{Bar--Sochen2019,Ben-Shaul--Bar--Fishelov--Sochen2023,Li--Zhai--Chen2021,Zhang--Li--Schutte2022,Liu--Dou--He--Yang--Jiang2023,Wang--Xie2023,Guo--Ming2023,Mattheakis--Schleder--Larson--Kaxiras2022,Liu--Dou--He--Yang--Jiang2023,
Jin--Mattheakis--Protopapas2020,Jin--Mattheakis--Protopapas2022,Holliday--Lindner--Ditto2023}.
We note that NeuralSVD is distinct from these regularization-based approaches: while regularization-based approaches are often susceptible to the choice of regularization parameters, NeuralSVD, which utilizes nesting techniques,  characterizes the ordered eigenbasis as its global optimizer without tuning any hyperparameter in the objective functions.

Other approaches include: \citet{Han--Lu--Zhou2020} proposes a stochastic differential equation framework that can learn the first mode of an eigenvalue problem; \citet{Yang--Deng--Yang--He--Zhang2023} propose a way to use neural networks for power and inverse power methods; \citet{Li--Ying2021} proposed a semigroup method for high dimensional elliptic PDEs and eigenvalue problems using neural networks.
More broadly, there exist other deep-learning-based solvers for general PDEs beyond EVP PDEs such as deep Ritz method~\citep{Yu--Weinan2018}, deep Galerkin method~\citep{Sirignano--Spiliopoulos2018}, and Fourier neural operator~\citep{Li--Kovachki--Azizzadenesheli--Liu--Bhattacharya--Stuart--Anandkumar2021}.

\vspace{.5em}\noindent\textbf{Quantum Chemistry Literature.}
Quantum chemistry has witnessed rapid recent advancements in this particular direction.
While the main problem in quantum chemistry is to solve the TISE of a given electronic system, the problem size grows rapidly: the domain has $3N$ dimension with $N$ electrons, and thus the complexity of solving TISE exponentially blows up even with $N$ of a moderate size. 
Therefore, the development in this domain has been focused on developing a new neural network architecture (called neural network \emph{ansatzes}) that better embed physical inductive bias for more expressivity.
Representative works include \citep{Carleo--Troyer2017}, SchNet~\citep{Schutt--Kindermans--Sauceda--Chmiela--Tkatchenko--Muller2017}, Fermionic neural networks~\citep{Choo--Mezzacapo--Carleo2020}, FermiNet~\citep{Pfau--Spencer--Matthews2020}, PauliNet~\citep{Hermann--Schatzle--Noe2020}, and DeepErwin~\citep{Gerard--Scherbela--Marquetand--Grohs2022}; see a comprehensive, recent review paper~\citep{Hermann--Spencer--Choo--Mezzacapo--Foulkes--Pfau--Carleo--Noe2022} for the overview of the field.

In most, if not all, of the works, the quantum Monte Carlo (QMC), also known as variational Monte Carlo (VMC)~\citep{Cuzzocrea--Scemama--Briels--Moroni--Filippi2020}, has been used as the de facto. 
QMC is essentially a special way to minimize the Rayleigh quotient to obtain the ground state energy.
Until recently, most of the works focused on the ground state (\ie the first bottom mode); a few recent exceptions are \citep{Entwistle--Schatzle--Erdman--Hermann--Noe2023,Pfau--Axelrod--Sutterud--von-Glehn--Spencer2023}, which proposed variations of QMC for excited states.
We believe that applying the proposed NestedLoRA framework to quantum chemistry problem can be an exciting research direction. 

\subsubsection{Spectral Pollution}
We remark in passing that in the numerical linear algebra literature, there exists a phenomenon called \emph{spectral pollution}, which refers to spurious eigenvalues introduced by discretizing an infinite-dimensional operator with respect to a fixed orthonormal basis~\citep{Davies--Plum2004}, and it is an active research area to address the issue; see, \eg \citep{Colbrook--Horning--Townsend2021} for a recent attempt. 
In principle, as our framework directly optimizes parametric eigenfunctions to fit the underlying eigenfunctions, we would not encounter such an issue due to discretization, provided that the parametric functions are sufficiently expressive. In experiments, we also did not observe any spurious eigenvalues with any of the methods for parametric eigenfunctions for the hydrogen atom. 

\subsection{In-Depth Review of SpIN and NeuralEF}
\label{app:sec:details}

\begin{table*}[t]
    \centering
    \caption{Comparison with SpIN~\citep{Pfau--Petersen--Agarwal--Barrett--Stachenfeld2019} and NeuralEF~\citep{Deng--Shi--Zhu2022}. 
    }  
    \begin{small}
    \setlength\tabcolsep{3 pt}
    \begin{tabular}{l c c c}
    \toprule
         & \textbf{SpIN} & \textbf{NeuralEF} & \textbf{NeuralSVD}\\
    \midrule        
        Goal & EVD & EVD & SVD/EVD\\
    \midrule
        \makecell[l]{(a) To handle %
        orthonormality constraints} & \makecell[c]{%
        Cholesky decomposition} & function normalization & - \\
        \makecell[l]{(b) To remove bias %
        in gradient estimation} & \makecell[t]{bi-level optimization;\\
        need to store Jacobian} & large batch size & - \\
    \bottomrule
    \end{tabular}
    \end{small}
    \label{tab:comparison}
\end{table*}

\subsubsection{SpIN}
\label{app:sec:spin}
Suppose that we wish to learn the top-$L$ eigenpairs of a linear operator $\Top$.
For trial eigenfunctions $\hat{\phi}_1,\ldots,\hat{\phi}_L$, define two $L\times L$ matrices $\Sigma$ and $\Pi$, which we call the gram matrix and the quadratic form matrix, respectively, as follows:
\[
\Sigma_{\ell\ell'}\defeq \langle \hat{\phi}_\ell | \hat{\phi}_{\ell'}\rangle
\quad\text{and}\quad
\Pi_{\ell\ell'}\defeq \langle \hat{\phi}_\ell | \Top \hat{\phi}_{\ell'}\rangle.
\]
Based on the trace maximization framework, we can solve the following optimization problem:
\[
\maximize_{\hat{\phi}_1,\ldots,\hat{\phi}_L\in\Fc} \tr(\Sigma^{-1} \Pi).
\]
Let $\Sigma=\Lm \Lm^\intercal$ be the Cholesky decomposition of $\Sigma$, where $\Lm$ is a lower-triangular matrix. Define $\Lambda\defeq \Lm^{-1}\Pi\Lm^{-\intercal} \in\Real^{L\times L}$.
By the property of trace, we can write
\begin{align*}
\tr(\Sigma^{-1}\Pi)
=\tr((\Lm\Lm^\intercal)^{-1}\Pi)
&= \tr(\Lm^{-1}\Pi\Lm^{-\intercal})
=\tr(\Lambda)
=\sum_{\ell=1}^L \Lambda_{\ell\ell}.
\end{align*}

\paragraph{Optimization with Masked Gradient.}
The key idea behind the SpIN optimization framework is in the following lemma.
\begin{lemma}
For each $\ell=1,\ldots,L$, 
$\Lambda_{\ell\ell}$ is only a function of $\hat{\phi}_1,\ldots,\hat{\phi}_\ell$.
\end{lemma}
\begin{proof}
It immediately follows from the upper triangular property of $\Lm^{-1}$.
\end{proof}
Assuming that $\hat{\phi}_1,\ldots,\hat{\phi}_{\ell-1}$ learn the top-$(\ell-1)$ eigen-subspace, SpIN updates $\hat{\phi}_\ell$ to only maximize $\Lambda_{\ell\ell}$, \ie based on the gradient $\partial_{\hat{\phi}_\ell} \Lambda_{\ell\ell}$ for each $\ell$.
Once optimized, the learned functions can be orthogonalized by $\Lm^{-1} \hat{\phiv}$.
Let $\Lm_\Sigma$ denote the Cholesky factor for a matrix $\Sigma$ and define 
\[
\Am_{\Sigma,\Pi}\defeq \Lm_{\Sigma}^{-\intercal}\mathsf{triu}(\Lm_\Sigma^{-1}\Pi\Lm_\Sigma^{-\intercal}\diag(\Lm_{\Sigma})^{-1}).
\]
Then, the \emph{masked} gradient can be collectively written as
\begin{align}
\label{eq:spin_gradient}
-\tilde{\partial}_{\th}\tr(\Lambda)
&= -\E_{p(x)}\Bigl[(\Top\hat{\phiv})(X)^\intercal \Lm_{\Sigma}^{-1}\diag(\Lm_{\Sigma})^{-1} \frac{\partial \hat{\phiv}(X)}{\partial{\th}}\Bigr]
+ \E_{p(x)}\Bigl[\hat{\phiv}(X)^\intercal \Am_{\Sigma,\Pi}
\frac{\partial \hat{\phiv}(X)}{\partial{\th}}\Bigr].
\end{align}
See eq.~(25) of \citep{Pfau--Petersen--Agarwal--Barrett--Stachenfeld2019} for the original expression with derivation.
A naive estimator of this gradient with minibatch samples would be to plug in the empirical (unbiased) estimates of $\Sigma$ and $\Pi$, which are
\[
\hat{\Sigma}\defeq \hat{\E}_{p(x)}[\hat{\phiv}(X)\hat{\phiv}(X)^\intercal]
\quad\text{and}\quad
\hat{\Pi}\defeq \hat{\E}_{p(x)}[
\hat{\phiv}(X)(\Top\hat{\phiv})(X)^\intercal].
\]
Note, however, the resulting gradient estimate is biased, since $\Lm_{\Sigma}$, $\Lm_{\Sigma}^{-1}$, and $\Sigma^{-1}$ are not linear in $\Sigma$.

\paragraph{Bi-level Stochastic Optimization for Unbiased Gradient Estimates.}
To detour the issue with the biased gradient estimate, 
\citet{Pfau--Petersen--Agarwal--Barrett--Stachenfeld2019} proposed to plug-in exponentially weighted moving average (EWMA) of two statistics into the expression, which can be understood as an instance of a bi-level stochastic optimization procedure with unbiased gradient estimates.
To motivate the approach, 
we rewrite the second term of \eqref{eq:spin_gradient} as
\begin{align*}
\E_{p(x)}\Bigl[\hat{\phiv}(X)^\intercal \Am_{\Sigma,\Pi}
\frac{\partial \hat{\phiv}(X)}{\partial{\th}}\Bigr]
&= \tr\Bigl(\Am_{\Sigma,\Pi}\E_{p(x)}\Bigl[
\frac{\partial \hat{\phiv}(X)}{\partial{\th}}
\hat{\phiv}(X)^\intercal
\Bigr]
\Bigr).
\end{align*}
Based on the expression, we maintain the EWMAs of $\Sigma$ and $\E_{p(x)}\Bigl[
\frac{\partial \hat{\phiv}(X)}{\partial{\th}}
\hat{\phiv}(X)^\intercal
\Bigr]$, 
which are denoted as $\bar{\Sigma}$ and $\bar{\mathbf{J}}$,
and updated via minibatch samples as follows:
\begin{align}
\bar{\Sigma}&\gets \b\bar{\Sigma} + (1-\b) \hat{\Sigma},\label{eq:spin_cov_update}\\
\bar{\Jb}&\gets \b\bar{\Jb} + (1-\b) \hat{\E}_{p(x)}\Bigl[
\frac{\partial \hat{\phiv}(X)}{\partial{\th}}
\hat{\phiv}(X)^\intercal
\Bigr].\label{eq:spin_jac_update}
\end{align}
Here $\b\in[0,1]$ is the decay parameter for EWMA.
Now, given these statistics, we update the parameter $\th$ by the following gradient estimate with minibatch samples:
\begin{align*}
-\widehat{\tilde{\partial}}_{\th}\tr(\Lambda)
&= -\hat{\E}_{p(x)}\Bigl[(\Top\hat{\phiv})(X)^\intercal \Lm_{\bar{\Sigma}}^{-1}\diag(\Lm_{\bar{\Sigma}})^{-1} \frac{\partial \hat{\phiv}(X)}{\partial{\th}}\Bigr]
+ \tr(\Am_{\bar{\Sigma},\hat{\Pi}}\bar{\Jb}).
\end{align*}
Note that the randomness in the second term is in $\hat{\Pi}$ and the second term is linear in $\hat{\Pi}$.
After all, the estimate is unbiased given $\bar{\Sigma}$ and $\bar{\Jb}$.

\paragraph{Discussion.}
SpIN is a pioneering work, being the first parametric framework to perform the top-$L$ EVD of a self-adjoint operator with parametric eigenfunctions.
However, the derivation of the masked gradient is rather involved, and the resulting algorithm's complexity is not favorably scaling in $L$.
In terms of the computational complexity, the Cholesky decomposition step that takes $O(L^3)$ for each iteration is not scalable in $L$.
Also, due to the bi-level stochastic optimization for unbiased gradient estimates, SpIN needs to maintain a separate copy of the Jacobian~\eqref{eq:spin_jac_update}, which may consume significant memory with large networks.
The decay parameter in the bi-level stochastic optimization is another sensitive hyperparameter to be tuned in the framework.
Finally, we remark that the idea of masked gradient is similar to the sequential nesting, and thus when it is applied to a shared parameterization, it cannot guarantee a desired optimization behavior.

\iftrue
\paragraph{SpIN-X.}
There exists a follow-up work of SpIN that proposed an alternative optimization method with several practical optimization techniques~\citep{Wu--Wang--Perdikaris2023}.
As the paper does not coin a term for the proposed method, we call it SpIN-X here.
The proposed method is based on the following modified objective function
\[
\Lc = \frac{1}{L} \Bigl\{
-w_0 \sum_{\ell=1}^L a_\ell \Lambda_{\ell\ell} + \sum_{\ell=1}^L w_\ell \| (\Top-\Lambda_{\ell\ell}\Iop) \hat{\phi}_\ell\|^2
\Bigr\}.
\]
Here, the weights $w_0,\ldots,w_L$ are defined as $w_\ell\defeq (K_0+\ldots+K_L)/K_\ell$, where $K_\ell \defeq \mathsf{sg}(\|\partial_\th \Lc_{\ell}\|_2)$, and $\Lambda_{\ell\ell}$ are eigenvalues still obtained from the Cholesky decomposition steps.
Though the experimental results in \citep{Jin--Mattheakis--Protopapas2022} show improved results over SpIN, some optimization techniques such as balanced gradients are nontrivial to apply, and thus we do not include a comparison with this approach.
\fi

\subsubsection{NeuralEF}
\label{app:sec:neuralef}
In essence, NeuralEF~\citep{Deng--Shi--Zhang--Cui--Lu--Zhu2022} starts from the following characterization of eigenfunctions, which can be understood as a sequential version of \eqref{eq:evd_char_subspace}.
\begin{proposition}\label{prop:sequential}
Let $\Top\suchthat \Fc\to\Fc$ be a linear, self-adjoint operator, where $\Fc$ is a Hilbert space.
Given functions $\phit_1,\ldots,\phit_{\ell-1}\in \Fc$, consider the optimization problem
\begin{align*}
(P_\ell)~~~
\begin{aligned}
\underset{\substack{\phit_\ell\in \Fc}}{\text{maximize~~}} & 
\langle \phit_\ell | \Top\phit_\ell\rangle \\
\text{subject~to~~} &
\langle \phit_\ell | \phit_i\rangle=\d_{\ell i}
\quad \forall 1\le i\le \ell.
\end{aligned}
\end{align*}
If $\phit_1,\ldots,\phit_{\ell-1}$ are the top $\ell-1$ eigenfunctions of the operator $\Top$, 
then the solution $\phit_\ell$ of the optimization problem $(P_\ell)$ is the $\ell$-th eigenfunction.
\end{proposition}
To avoid the explicit orthogonality constraint, NeuralEF proposes to solve the following optimization problem, generalizing the formulation of EigenGame~\citep{Gemp--McWilliams--Vernade--Graepel2021} for operators: 
\begin{align}
\label{eq:neuralef_problem}
(P_\ell')~~~
\begin{aligned}
\underset{\substack{\phit_\ell\in \Fc}}{\text{minimize~~}} & 
\Lc_{\ell}(\phiv_{1:\ell})\defeq -\langle \phit_\ell | \Top\phit_\ell\rangle +\sum_{i=1}^{\ell-1} \frac{\langle \phit_\ell | \Top \phit_i\rangle}{\langle \phit_i | \Top \phit_i\rangle} \\
\text{subject~to~~} &
\langle \phit_\ell | \phit_\ell\rangle=1.
\end{aligned}
\end{align}
Replacing the constraints $\langle \phit_\ell | \phit_i\rangle=0$ as $\langle \phit_i |  \Top\phit_i\rangle^2=0$ for each $1\le i\le \ell-1$, we can view $(P_\ell')$ as a relaxed optimization problem of $(P_\ell)$. Here, $\langle \phit_i|\Top\phit_i\rangle^{-1}$ plays the role of a Lagrangian multiplier for the $i$-th constraint.
With this specific choice of weights, this partially unconstrained optimization problem has the same guarantee (Proposition~\ref{prop:sequential}) for $(P_\ell)$ as follows:
\begin{theorem}
If $\phit_1,\ldots,\phit_{\ell-1}$ are the top $\ell-1$ eigenfunctions $\phi_1,\ldots,\phi_{\ell-1}$ of $\Top$, 
then the solution $\phit_\ell$ of the optimization problem $(P_\ell')$ is the $\ell$-the eigenfunction.
\end{theorem}
\begin{proof}[Informal proof]
For the sake of simplicity, we assume that there are only $m$ finite eigenvalues $\lambda_1,\ldots,\lambda_m$ and $\ell+1\le m$.
Let $\phi_1,\phi_2,\ldots$ be the eigenfunctions of $K$ which form an othornormal basis of $\Lc_{p(x)}^2(\Xc)$.
We first write a function $\phit_\ell$ as a linear combination of the eigenfunctions 
\[
\phit_\ell(x)=\sum_{i=1}^\infty \langle \phit_\ell | \phi_i \rangle.
\]
Then, we can readily observe that
\begin{align*}
\langle \phit_\ell | \Top \phit_\ell\rangle 
&= \sum_{i=1}^m \lambda_i \langle \phit_\ell | \phi_i\rangle^2,\\
\langle \phit_\ell|\Top\phi_i\rangle &= \lambda_i \langle \phit_\ell | \phi_i\rangle,\\
\langle \phi_i | \Top\phi_i\rangle &=\lambda_i.
\end{align*}
Therefore, the objective becomes
\begin{align*}
-\langle \phit_\ell|\Top\phit_\ell\rangle +\sum_{i=1}^{\ell-1} \frac{\langle \phit_\ell| \Top\phi_i\rangle^2}{\langle \phi_i| \Top\phi_i\rangle} 
&= -\sum_{i=1}^m \lambda_i \langle \phit_\ell,\phi_i\rangle^2 +  \sum_{i=1}^{\ell-1} \lambda_i \langle \phit_\ell,\phi_i\rangle^2\\
&= -\sum_{i=\ell}^m \lambda_i \langle \phit_\ell,\phi_i\rangle^2,
\end{align*}
which implies that the objective is uniquely minimized when $\langle \phit_\ell,\phi_i\rangle=\d_{i\ell}$ for $i\ge \ell$, \ie when $\phit_\ell$ is the $\ell$-th eigenfunction $\phi_\ell$.
\end{proof}

Hence, solving the sequence of optimization problems $(P_\ell')$ leads to finding the eigenfunctions in order.
\redt{%
To emulate to solve the sequential optimization, \citet{Deng--Shi--Zhang--Cui--Lu--Zhu2022} proposed to solve}
\begin{align*}
\begin{aligned}
\underset{\substack{\phit_1,\ldots,\phit_L\in \Fc}}{\text{minimize~~}} & 
\sum_{\ell=1}^L \Bigl\{
-\langle \phit_\ell|\Top\phit_\ell\rangle +\sum_{i=1}^{\ell-1} \frac{\langle \phit_\ell|\Top\sg(\phit_i)\rangle^2}{\langle \sg(\phit_i)|\Top\sg(\phit_i)\rangle}
\Bigr\}\\
\text{subject~to~~} &
\langle\phit_\ell | \phit_\ell\rangle=1
\text{~for~} 1\le \ell\le L.
\end{aligned}
\end{align*}
Here, $\sg$ denotes the stop-gradient operation, and thus this is not a properly defined optimization problem, rather defining an optimization procedure.
It is worth emphasizing that the minimization procedure no longer guarantees a structured solution if the stop gradient operations are removed.
To satisfy the normalization constraints, NeuralEF uses the $L_2$-batch normalization during training.
\begin{align}
\partial_{\phi_\ell} \Lc_{\ell}(\phiv_{1:\ell})
=4\Bigl\{-|\Top\phi_\ell\rangle
+\sum_{i=1}^{\ell-1}  \frac{\langle \phi_i| \Top\phi_\ell\rangle}{\langle \phi_i| \Top\phi_i\rangle}  |\Top \phi_i\rangle\Bigr\}.
\label{eq:neuralef}
\end{align}

\paragraph{Discussion.}
NeuralEF improves SpIN in general, providing a simpler optimization procedure, \ie without the costly Cholesky decomposition steps and the Jacobian updates.
The game-theoretic formulation that stemmed from EigenGame~\citep{Gemp--McWilliams--Vernade--Graepel2021} is similar to the idea of sequential nesting, and it might be problematic when applied to a shared parameterization as the sequential nesting is. 
The crucial difference of NeuralEF is that the $\ell$-th objective of NeuralEF has a guarantee only if the previous $(\ell-1)$ eigenfunctions are well learned, 
whereas the LoRA objective can characterize the eigensubspace and thus we can apply the joint nesting for a shared parameterization.
Moreover, NestedLoRA can naturally handle SVD.

\vspace{.5em}\noindent\textbf{An Unbiased-Gradient Variation.}
We note that in the streaming PCA literature, 
the authors of EigenGame~\citep{Gemp--McWilliams--Vernade--Graepel2021} proposed an unbiased variant of the original EigenGame in their subsequent work~\citep{Gemp--McWilliams--Vernade--Graepel2022}.
Following the same idea, one can easily think of an unbiased variant of NeuralEF, which corresponds to the following gradient:
\begin{align*}
\partial_{\phi_\ell} \Lc_{\ell}(\phiv_{1:\ell})
&=4\Bigl\{-|\Top\phi_\ell\rangle
+\sum_{i=1}^{\ell-1} \langle \phi_i|\Top \phi_\ell\rangle | \phi_i\rangle\Bigr\}.%
\end{align*}
We can also consider another variant of similar flavor
\begin{align*}
\partial_{\phi_\ell} \Lc_{\ell}(\phiv_{1:\ell})
&=4\Bigl\{-|\Top\phi_\ell\rangle
+\sum_{i=1}^{\ell-1}  {\langle \phi_i| \phi_\ell\rangle}  |\Top \phi_i\rangle\Bigr\}. 
\end{align*}
Both variants are implemented in our codebase.
In our experiment, we used the first variant instead of the original~\eqref{eq:neuralef}, as we found that the original NeuralEF performs much worse than the variant.
In the current manuscript, therefore, we show that NeuralSVD can even outperform the improved version of NeuralEF.

\section{Technical Details and Deferred Proofs}

\subsection{Derivation of the Low-Rank Approximation Objective}
\label{app:sec:lora_derivation}
Recall that we define the LoRA objective as
\begin{align*}
\loraobj(\fv_{1:L}, \gv_{1:L})
&\defeq -2\sum_{\ell=1}^L \langle g_\ell|\Top f_\ell\rangle + \sum_{\ell=1}^L\sum_{\ell'=1}^L \langle f_\ell|f_{\ell'}\rangle \langle g_\ell|g_{\ell'}\rangle.
\end{align*}
When $\Top$ is a compact operator, the LoRA objective can be derived as the approximation error of $\Top$ via a low-rank expansion $\sum_{\ell=1}^L |f_\ell\rangle \langle g_\ell|$ measured in the squared Hilbert--Schmidt norm.
For a linear operator $\Top\suchthat\Fc\to\Gc$ for Hilbert spaces $\Fc$ and $\Gc$, 
the \emph{Hilbert--Schmidt norm} $\|\Top\|_{\mathsf{HS}}$ of an operator $\Top$ is defined as
\[
\| \Top\|_{\mathsf{HS}}^2 \defeq \sum_{i\in I} \|\Top\phi_i\|^2
\]
for an orthonormal basis $\{\phi_i\suchthat i\in I\}$ of the Hilbert space $\Fc$.
Note that the Hilbert--Schmidt norm is well-defined in that it is independent of the choice of the orthonormal basis.
When $\Fc$ and $\Gc$ are finite-dimensional, \ie when $\Top$ is a matrix, it boils down to the Frobenius norm.
When $\|\Top\|_{\mathsf{HS}}<\infty$, $\Top$ is said to be \emph{compact}.
\begin{lemma}
\label{lem:lora}
If $\Top$ is compact, then
\begin{align}
\loraobj(\fv_{1:L}, \gv_{1:L})
&= \Bigl\|\Top - \sum_{\ell=1}^L |g_\ell\rangle \langle f_\ell|\Bigr\|_{\mathsf{HS}}^2 - \|\Top\|_{\mathsf{HS}}^2
\label{eq:lora_obj_motivation}.
\end{align}
\end{lemma}
\begin{proof}
Pick an orthonormal basis $\{\phi_i\suchthat i\in I\}$ of $\Top$.
Note that $\sum_{i\in I} |\phi_i\rangle \langle \phi_i|=\Iop$, where $\Iop$ denotes the identity operator.
Hence, we have
\begin{align*}
\Bigl\|\Top - \sum_{\ell=1}^L |g_\ell\rangle \langle f_\ell|\Bigr\|_{\mathsf{HS}}^2 - \|\Top\|_{\mathsf{HS}}^2
&= \sum_{i\in I} \Bigl\|\Top|\phi_i\rangle - \sum_{\ell=1}^L |g_\ell\rangle \langle f_\ell|\phi_i\rangle\Bigr\|^2 - \sum_{i\in I} \|\Top|\phi_i\rangle\|^2\\
&= \sum_{i\in I}\Bigl(
-2\sum_{\ell=1}^L \langle f_\ell|\phi_i\rangle \langle g_\ell|\Tc|\phi_i\rangle + \sum_{\ell=1}^L \sum_{\ell'=1}^L \langle f_\ell|\phi_i\rangle \langle f_{\ell'}|\phi_i\rangle \langle g_\ell|g_{\ell'}\rangle
\Bigr)\\
&= -2\sum_{\ell=1}^L \Bigl\langle g_\ell \Big|\Tc\Bigl(\sum_{i\in I}|\phi_i\rangle \langle \phi_i|\Bigr) \Bigr|f_\ell\Bigr\rangle
+ \sum_{\ell=1}^L \sum_{\ell'=1}^L \Bigl\langle f_\ell\Big|\Bigl(\sum_{i\in I}|\phi_i\rangle \langle \phi_i|\Bigr)\Big|f_{\ell'}\Bigr\rangle \langle g_\ell|g_{\ell'}\rangle\\
&= -2\sum_{\ell=1}^L \langle g_\ell|\Top f_\ell\rangle + \sum_{\ell=1}^L\sum_{\ell'=1}^L \langle f_\ell|f_{\ell'}\rangle \langle g_\ell|g_{\ell'}\rangle\\
&=\loraobj(\fv_{1:L}, \gv_{1:L}).\qedhere
\end{align*}
\end{proof}
Though this relationship~\eqref{eq:lora_obj_motivation} holds only for a compact operator, we remark that the LoRA objective~\eqref{eq:lora_obj} is well-defined for any operator $\Top$.

\subsection{Proof of Theorem~\ref{thm:subspace}}
\label{app:sec:proof_thm:subspace}
From Lemma~\ref{lem:lora},
Theorem~\ref{thm:subspace} follows as a corollary of Schmidt's LoRA theorem, stated below. \qed
\begin{theorem}[{\citet{Schmidt1907}}]
\label{thm:schmidt}
Suppose that $\Top\suchthat\Fc\to\Gc$ is a compact operator with $\{(\sigma_{\ell}, f_\ell, g_\ell)\}_{\ell=1}^{\infty}$ as its singular triplets.
Define
\[
(\fv^\star,\gv^\star)
\defeq 
\argmin_{\substack{f_\ell\in \Fc, g_\ell\in \Gc, \ \ell\in[L]}} \Bigl\|\Top - \sum_{\ell=1}^L |g_\ell\rangle \langle f_\ell|\Bigr\|_{\mathsf{HS}}^2.
\]
If $\sigma_L >\sigma_{L+1}$, we have
\[
\sum_{\ell=1}^L |g_\ell^\star\rangle \langle f_\ell^\star|
=\sum_{\ell=1}^L \sigma_\ell |\psi_\ell\rangle \langle\phi_\ell|.
\]
\end{theorem}

\subsection{Proof of Theorem~\ref{thm:seq_nesting} (Sequential Nesting)}
\label{app:sec:proof_thm:seq_nesting}
Recall that we assume
\[
\sum_{i=1}^{\ell-1} |g_i\rangle \langle f_i|
= \sum_{i=1}^{\ell-1} \sigma_i |\psi_i\rangle \langle \phi_i|.
\]
By Lemma~\ref{lem:lora}, the LoRA objective can be written as
\begin{align*}
\loraobj(\fv_{1:\ell}, \gv_{1:\ell})
&= \Bigl\|\Top - \sum_{i=1}^\ell |g_i\rangle \langle f_i|\Bigr\|_{\mathsf{HS}}^2 - \|\Top\|_{\mathsf{HS}}^2\\
&= \Bigl\|\sum_{i\ge 1} \sigma_i |\psi_i\rangle\langle \phi_i| - \sum_{i=1}^{\ell-1} \sigma_i |\psi_i\rangle \langle \phi_i| -|g_\ell\rangle \langle f_\ell|\Bigr\|_{\mathsf{HS}}^2 - \|\Top\|_{\mathsf{HS}}^2\\
&= \Bigl\|\sum_{i\ge \ell} \sigma_i |\psi_i\rangle\langle \phi_i|  -|g_\ell\rangle \langle f_\ell|\Bigr\|_{\mathsf{HS}}^2 - \|\Top\|_{\mathsf{HS}}^2.
\end{align*}
Hence, minimizing $\loraobj(\fv_{1:\ell}, \gv_{1:\ell})$ with respect to $(f_\ell,g_\ell)$ is equivalent to minimizing the LoRA objective $\loraobj(f_\ell,g_\ell;\Tc_{\ge\ell})$ defined with respect to the truncated operator $\Tc_{\ge\ell}\defeq \sum_{i\ge \ell}\sigma_i |\psi_i\rangle\langle \phi_i|$.
Since $\sigma_\ell>\sigma_{\ell+1}$, by Schmidt's theorem (Theorem~\ref{thm:schmidt}), the global optimizer must satisfy  $|g_\ell\rangle \langle f_\ell|
= \sigma_\ell |\psi_\ell\rangle \langle \phi_\ell|$.
\qed

\subsection{Proof of Theorem~\ref{thm:jnt_nesting} (Joint Nesting)}
\label{app:sec:proof_thm:jnt_nesting}

We first prove the following lemma; Theorem~\ref{thm:jnt_nesting} readily follows as a corollary.

\begin{lemma}\label{lem:jnt_nesting}
Suppose that all the nonzero singular values of the target kernel are distinct.
If $\sigma_\ell > \sigma_{\ell +1}$,
the objective function $\tilde{\Lc}(\fv,\gv)\defeq \loraobj(\fv_{1:\ell},\gv_{1:\ell})
+ w\loraobj(\fv_{[L]},\gv_{[L]})$ with $w>0$ is minimized if and only if
\begin{align*}
\sum_{i=1}^\ell |g_i^*\rangle \langle f_i^*|
&= \sum_{i=1}^\ell \sigma_i |\psi_i\rangle \langle \phi_i|
\quad\text{ and }\quad
\sum_{i=\ell+1}^L |g_i^*\rangle \langle f_i^*|
= \sum_{i=\ell+1}^L \sigma_i |\psi_i\rangle\langle \phi_i|.
\end{align*}
\end{lemma}

\begin{proof}
First, note that by the Schmidt  theorem (Theorem~\ref{thm:schmidt}), 
\[
\loraobj(\fv_{1:\ell}, \gv_{1:\ell})
\ge \loraobj(\fv_{1:\ell}^*, \gv_{1:\ell}^*)
= \sum_{i=1}^\ell \sigma_\ell ^2,
\]
where the equality holds if and only if
\[
\sum_{i=1}^\ell |g_i^*\rangle \langle f_i^*|= \sum_{i=1}^\ell \sigma_i |\psi_i\rangle\langle \phi_i|.
\]
Using this property, 
we immediately have a lower bound
\[
\loraobj(\fv_{1:\ell}, \gv_{1:\ell})
+w \loraobj(\fv_{1:L}, \gv_{1:L})
\ge \sum_{i=1}^\ell\sigma_i^2 + w \sum_{j=1}^L\sigma_j^2,
\]
where the equality holds if and only if
\begin{align*}
\sum_{i=1}^\ell |g_i^*\rangle \langle f_i^*|
&= \sum_{i=1}^\ell \sigma_i |\psi_i\rangle\langle \phi_i|
\quad\text{ and }\quad
\sum_{i=1}^L |g_i^*\rangle \langle f_i^*|
= \sum_{i=1}^L \sigma_i |\psi_i\rangle\langle \phi_i|,
\end{align*}
which is equivalent to
\begin{align*}
\sum_{i=1}^\ell |g_i^*\rangle \langle f_i^*|
= \sum_{i=1}^\ell \sigma_i |\psi_i\rangle\langle \phi_i|
\quad\text{ and }\quad
\sum_{i=\ell+1}^L |g_i^*\rangle \langle f_i^*|
&= \sum_{i=\ell+1}^L \sigma_i |\psi_i\rangle\langle \phi_i|.\qedhere
\end{align*}
\end{proof}

We are now ready to prove Theorem~\ref{thm:jnt_nesting}.

\begin{proof}[Proof of Theorem~\ref{thm:jnt_nesting}]
By inductively applying Lemma~\ref{lem:jnt_nesting} to the grouping $\jntloraobj(\fv,\gv) = \sum_{i=1}^\ell w_i \loraobj(\fv_{[i]}, \gv_{[i]}) 
+ \sum_{i=\ell + 1}^L w_i \loraobj(\fv_{[i]}, \gv_{[i]})$ for $\ell=1,\ldots, L-1$, a minimizer must satisfy 
\[
\sum_{i=1}^\ell |g_i^*\rangle \langle f_i^*|= \sum_{i=1}^\ell \sigma_i |\psi_i\rangle\langle \phi_i|,
\]
for each $\ell=1,\ldots,L$. 
This implies that the equivalence $|g_i^*\rangle \langle f_i^*|= \sigma_i |\psi_i\rangle\langle \phi_i|$ should hold term by term.
\end{proof}

\subsection{One-Shot Computation of Jointly Nested Objective}
\label{sec:app:oneshot}

The gradient of the joint nesting objective~\eqref{eq:jnt_nesting_grad} can be computed based on the following observation:
\begin{proposition}[One-shot computation]
\label{prop:oneshot}
Given a positive weight vector $\weights$, 
define $\vectormask\in \Real^L$ and $\matrixmask\in \Real^{L\times L}$ as
$\mathsf{m}_i \defeq \sum_{\ell=i}^L w_\ell$ and $\matrixmask_{ij} \defeq \mathsf{m}_{\max\{i,j\}}$.
Then, the nested objective is written as
\begin{align*}
\jntloraobj(\fv,\gv;\weights) 
&\defeq -2\sum_{\ell=1}^L \mathsf{m}_\ell \langle g_\ell|\Top f_\ell\rangle + \sum_{\ell=1}^L\sum_{\ell'=1}^L \matrixmask_{\ell\ell'} \langle f_\ell|f_{\ell'}\rangle \langle g_\ell|g_{\ell'}\rangle.
\end{align*}
\end{proposition}

\begin{proof}
Recall that
\begin{align*}
\jntloraobj(\fv,\gv;\weights)
&= \sum_{\ell=1}^L w_\ell \loraobj(\fv_{1:\ell},\gv_{1:\ell})\\
&= \sum_{\ell=1}^L w_\ell 
\Bigl\{
-2\sum_{i=1}^\ell \langle g_i|\Top f_i\rangle + \sum_{i=1}^\ell\sum_{j=1}^\ell \langle f_i|f_j\rangle \langle g_i|g_j\rangle
\Bigr\}
\end{align*}
For the first term, we can write
\begin{align*}
\sum_{\ell=1}^L w_\ell \sum_{i=1}^\ell \langle g_i|\Top f_i\rangle
&= \sum_{\ell=1}^L w_\ell \sum_{i=1}^\ell f_i(x)g_i(y)
= \sum_{\ell=1}^L \mathsf{m}_\ell \langle g_\ell|\Top f_\ell\rangle,
\end{align*}
where $\mathsf{m}_\ell \defeq \sum_{i=\ell}^L w_i$.
For the second term, we can write
\begin{align*}
\sum_{\ell=1}^L w_\ell \sum_{1\le i,j\le \ell} \langle f_i|f_j\rangle \langle g_i|g_j\rangle
&= \sum_{1\le i,j\le L} \Mm_{ij} \langle f_i|f_j\rangle \langle g_i|g_j\rangle,
\end{align*}
where $\matrixmask_{ij} \defeq \mathsf{m}_{\max\{i,j\}} = \sum_{\ell=\max\{i,j\}}^L w_{\ell}$.
This concludes the proof.
\end{proof}

\subsection{EVD with Non-Compact Operators}
\label{app:sec:noncompact}
For a self-adjoint operator $\Top$, we can apply our framework by considering the induced LoRA objective
\[
\loraobj(\fv_{1:L})
\defeq -2\sum_{\ell=1}^L \langle f_\ell|\Top f_\ell\rangle + \sum_{\ell=1}^L\sum_{\ell'=1}^L \langle f_\ell|f_{\ell'}\rangle^2.
\]

Though the original LoRA theorem of Schmidt (Theorem~\ref{thm:schmidt}) holds for a compact operator, 
it can be extended to a certain class of non-compact operators, which have discrete eigenvalues.

\begin{theorem}
\label{thm:schmidt_evd_noncompact}
For a self-adjoint operator $\Top$, define
\[
\fv^\star
\defeq 
\argmin_{\substack{f_\ell\in \Fc,\ \ell\in[L]}} \loraobj(\fv).
\]
Suppose that the operator $\Top$ has $r$ positive eigenvalues $\lambda_1\ge \ldots\ge \lambda_r >0 \ge \lambda_{r+1}\ge \ldots$ with corresponding orthonormal eigenfunctions $\{\phi_\ell\}_{\ell\ge 1}$, for some $r\in\Natural\cup\{\infty\}$.
If $r<\infty$, the span of $|\fv^\star\rangle$ is equal to the span of the top-$\min\{L,\rk\}$ eigenfunctions of the operator $\Top$, or more precisely
\[
\sum_{\ell=1}^L |f_\ell^\star\rangle \langle f_\ell^\star|
=\sum_{\ell=1}^{\min\{L,\rk\}} \lambda_\ell |\phi_\ell\rangle \langle\phi_\ell|.
\]
If $r=\infty$, \ie when there are countably infinitely many positive eigenvalues, the same holds if $\lambda_L > \lambda_{L+1}$.
\end{theorem}

As a consequence of this theorem, when we optimize the $|\fv_{1:L}\rangle$ with nesting for $L>\rk$, one can easily show that the optimal $|f_{r+1}^\star\rangle,\ldots,|f_{L}^\star\rangle$ are zero functions; we omit the proof.

\begin{proof}
We first consider when $r$ is finite.
Define the positive part of the operator as
\[
\Top_+ \defeq \sum_{\ell=1}^\rk \lambda_\ell |\phi_\ell\rangle \langle\phi_\ell|,
\]
which is compact by definition.
Then, $\Top_+-\Top$ is PSD with eigenvalues $0\le -\lambda_{r+1} \le -\lambda_{r+2}\le \ldots$ and eigenfunctions $\{\phi_\ell\}_{\ell\ge r+1}$.
Then, we can rewrite and lower bound the LoRA objective as
\begin{align*}
\loraobj(\fv_{1:L})
&= \Bigl\|\Top_+ - \sum_{\ell=1}^L |f_\ell\rangle \langle f_\ell|\Bigr\|_{\mathsf{HS}}^2 - \|\Top_+\|_{\mathsf{HS}}^2
+2\sum_{\ell=1}^L \langle f_\ell|(\Top_+ - \Top) f_\ell\rangle\\
&\stackrel{(a)}{\ge}\Bigl\|\Top_+ - \sum_{\ell=1}^L |f_\ell\rangle \langle f_\ell|\Bigr\|_{\mathsf{HS}}^2 - \|\Top_+\|_{\mathsf{HS}}^2,
\end{align*}
where the inequality $(a)$ follows since $\Top_+-\Top$ is PSD.
We note that the lower bound is minimized if and only if the span of $|\fv^\star\rangle$ is equal to the span of the top-$\min\{L,\rk\}$ eigenfunctions of the operator $\Top_+$ by applying Schmidt's theorem (Theorem~\ref{thm:schmidt}).
We further note that $(a)$ holds with equality, as $|f_1^\star\rangle,\ldots,|f_L^\star\rangle$ belong to the null space of $\Top_+-\Top$.
Hence, this concludes that the LoRA objective is minimized if and only if the span of $|\fv^\star\rangle$ is equal to the span of the top-$\min\{L,\rk\}$ eigenfunctions of the operator $\Top$.

When $r=\infty$, given that $\lambda_L >\lambda_{L+1}$, the rank-$L$ approximation of $\Top$
\[
\Top_L\defeq \sum_{\ell=1}^L \lambda_\ell|\phi_\ell\rangle \langle\phi_\ell|
\]
is well-defined.
The same proof for $r<\infty$ is valid if we replace $\Top_+$ with $\Top_L$.
\end{proof}

\section{Implementation Details with Code Snippets}
\label{app:sec:pseudo}
In this section, we explain how to implement the proposed NestedLoRA updates, providing readily deployable code snippets written in PyTorch. 
These code snippets are simplified from the actual implementation \ifforreview in the Supplementary Material \else which can be found online\footnote{\url{https://github.com/jongharyu/neural-svd}} \fi for the ease of exposition.

\subsection{Helper Functions: Computing Nesting Masks and Metric Loss}
\label{app:sec:helper}
As noted in \S\ref{sec:jnt_nesting}, both versions of NestedLoRA can be implemented in a unified way via the nesting masks $(\mm_\ell)_{\ell\in[L]}$ and $(\Mm_{i\ell})_{i\in[L],\ell\in[L]}$. 
Recall that for joint nesting, given positive weights $(w_1,\ldots,w_L)$, we define $\mm_\ell\defeq \sum_{i=\ell}^L w_i$ and $\Mm_{i\ell}\defeq \mm_{\max\{i,\ell\}}$.
\begin{lstlisting}[]
def get_joint_nesting_masks(weights: np.ndarray, set_first_mode_const: bool = False):
    vector_mask = list(np.cumsum(list(weights)[::-1])[::-1])
    if set_first_mode_const:
        vector_mask = [vector_mask[0]] + vector_mask
    vector_mask = torch.tensor(np.array(vector_mask)).float()
    matrix_mask = torch.minimum(vector_mask.unsqueeze(1), vector_mask.unsqueeze(1).T).float()
    return vector_mask, matrix_mask
\end{lstlisting}
Here, when the argument \texttt{set\_first\_mode\_const} is set to be \texttt{True}, it outputs masks for CDK, for which we explicitly add the constant first mode; see \S\ref{app:sec:cdk_implementation}.

The sequential nesting~\eqref{eq:seq_nesting} can be implemented by defining $\mm_\ell\defeq 1$ and $\Mm_{i\ell}\defeq \ones\{i\le \ell\}$.
\begin{lstlisting}[]
def get_sequential_nesting_masks(L, set_first_mode_const: bool = False):
    if set_first_mode_const:
        L += 1
    vector_mask = torch.ones(L)
    matrix_mask = torch.triu(torch.ones(L, L))
    return vector_mask, matrix_mask
\end{lstlisting}

In the LoRA objective~\eqref{eq:lora_obj}, the second term (with nesting), which we call the \emph{metric loss},
\begin{align}
\label{eq:metric_loss}
\sum_{\ell=1}^L \sum_{\ell'=1}^L \Mm_{\ell\ell'} \langle f_\ell|f_{\ell'}\rangle \langle g_\ell|g_{\ell'}\rangle
=\sum_{\ell=1}^L \sum_{\ell'=1}^L (\Mm\odot  \Lambda_{\fv} \odot \Lambda_{\gv})_{\ell\ell'}
\end{align}
is independent of the operator,
where we define $(\Lambda_{\fv})_{\ell\ell'}\defeq \langle f_\ell|f_{\ell'}\rangle$ for $\ell,\ell'\in[L]$ and $\odot$ denotes the elementwise matrix product.
Given samples $x_1,\ldots,x_B$, we can estimate each entry of the matrix $\Lambda_\fv\in\Real^{L\times L}$ by
\begin{align*}
(\hat{\Lambda}_{\fv})_{\ell\ell'}\defeq \frac{1}{B}\sum_{b=1}^B f_\ell(x) f_{\ell'}(x),
\end{align*}
which can be computed with PyTorch as:
\begin{lstlisting}[]
def compute_lambda(f):
    return torch.einsum('bl,bm->lm', f, f) / f.shape[0]  # (L, L)
\end{lstlisting}
Then, the metric loss can be computed as follows:
\begin{lstlisting}[]
def compute_loss_metric(f, g, matrix_mask):
    lam_f = compute_lambda(f)
    lam_g = compute_lambda(g)
    # compute loss_metric = E_{p(x)p(y)}[(f^T(x) g(y))^2]
    # f: (B1, L)
    # g: (B2, L)
    # lam_f, lam_g: (L, L)
    return (matrix_mask * lam_f * lam_g).sum(), lam_f, lam_g  # O(L ** 2)
\end{lstlisting}
Note that this metric loss needs not be computed when computing gradients.

\subsection{NestedLoRA Gradient Computation for Analytical Operators}
In this section, we explain how to implement the NestedLoRA gradient updates for analytical operators.
For the sake of simplicity, we explain for the implementation for EVD; the implementation for SVD can be found in our official PyTorch implementation.

For EVD of a self-adjoint operator $\Tc$, identifying $\gv$ with $\fv$, we need to compute the gradient
\begin{align}
(\partial_{f_\ell} \Lc)(x_b)
=2\Bigl\{-\mm_\ell(\Top f_\ell)(x_b)
+\sum_{i=1}^L \matrixmask_{i\ell} f_i(x_b) \langle f_i| f_\ell\rangle  \Bigr\}
\label{eq:nesting_grad_evd}
\end{align}
for each $\ell\in[L]$; see \eqref{eq:jnt_nesting_grad} for the general expression.
We can compute the gradient in an unbiased manner by plugging in the unbiased estimate of $\Lambda_{\fv}$ based on $\{x_1',\ldots,x_{B'}'\}$. We remark that the minibatch samples for estimating $\Lambda_{\fv}$ needs to be independent to $\{x_1,\ldots,x_B\}$ so that the overall gradient estimate for $\langle \partial_\th f_\ell|\partial_{f_\ell}\Lc\rangle$ becomes unbiased.
This can be efficiently implemented in a vectorized manner by writing a custom backward function with the automatic differentiation package of PyTorch as follows.
In what follows, we assume that $\{\fv(x_b)\}_{b=1}^B$ and $\{(\Tc \fv)(x_b)\}_{b=1}^B$ are already computed for a given $\fv$ and provided as \texttt{f} and \texttt{Tf}, resepectively.
Further, \texttt{f1} and \texttt{f2} must be independent to each other.

\begin{lstlisting}
class NestedLoRALossFunctionEVD(torch.autograd.Function):
    @staticmethod
    @torch.cuda.amp.custom_fwd
    def forward(
            ctx: torch.autograd.function.FunctionCtx,
            f,
            Tf,
            f1,
            f2,
            vector_mask,
            matrix_mask,
    ):
        """
        the reduction assumed here is `mean` (i.e., we take average over batch)
            f: (B, L) or (B, L, O)
            Tf: (B, L) or (B, L, O)
            f1: (B1, L) or (B1, L, O)
            f2: (B2, L) or (B2, L, O)
        warning: f1 and f2 must be independent
        """
        ctx.vector_mask = vector_mask = vector_mask.to(f.device)
        ctx.matrix_mask = matrix_mask = matrix_mask.to(f.device)
        loss_metric, lam_f1, lam_f2 = compute_loss_metric(f1, f2, matrix_mask)
        ctx.save_for_backward(f, Tf, f1, f2, lam_f1, lam_f2)
        # compute loss_operator = -2 * E_{p(x)}[\sum_{l=1}^L f_l^T(x) (Tf_l)(x)]
        loss_operator = - 2 * torch.einsum('l,bl,bl->b', vector_mask, f, Tf).mean()  # O(B1 * L * O)
        loss = loss_operator + loss_metric
        return loss

    @staticmethod
    @torch.cuda.amp.custom_bwd
    def backward(
            ctx: torch.autograd.function.FunctionCtx,
            grad_output: torch.Tensor
    ) -> Tuple[torch.Tensor, ...]:
        """
        Args:
            ctx: The context object to retrieve saved tensors
            grad_output: The gradient of the loss with respect to the output
        """
        f, Tf, f1, f2, lam_f1, lam_f2 = ctx.saved_tensors
        operator_f = - (4 / f.shape[0]) * torch.einsum('l,bl->bl', ctx.vector_mask, Tf)
        metric_f1 = (2 / f1.shape[0]) * torch.einsum('lm,lm,bl->bm', ctx.matrix_mask, lam_f2, f1)
        metric_f2 = (2 / f2.shape[0]) * torch.einsum('lm,lm,bl->bm', ctx.matrix_mask, lam_f1, f2)
        return grad_output * operator_f, None, grad_output * metric_f1, grad_output * metric_f2, \
            None, None, None
\end{lstlisting}

In practice, when given a minibatch $\{x_1,\ldots,x_{B}\}$, we can use the entire batch to compute \texttt{f} and \texttt{Tf}, and split \texttt{f} into two equal parts and plug in them to \texttt{f1} and \texttt{f2} to ensure the independence.
In what follows, the operator is given as an abstract function \texttt{operator}, whose interface is explained in the next section (\S\ref{app:sec:importance_sampling}).
\begin{lstlisting}[]
def compute_loss_operator(
        model,
        operator,
        x,
        importance=None,
):
    Tf, f = operator(model, x, importance=importance)
    f1, f2 = torch.chunk(f, 2)
    loss = NestedLoRALossFunctionEVD.apply(
        f, Tf, f1, f2,
        vector_mask,
        matrix_mask,
    )
    return loss, dict(f=f, Tf=Tf, eigvals=None)
\end{lstlisting}
After this function returns \texttt{loss}, calling \texttt{loss.backward()} will backpropagate the gradients based on the custom backward function, and populate the gradient for each model parameter.

\subsection{Importance Sampling}
\label{app:sec:importance_sampling}
\newcommand{\ptr}{p_{\mathsf{tr}}}
\newcommand{\pte}{p_{\mathsf{te}}}
\newcommand{\wtr}{w_{\mathsf{tr}}}
\newcommand{\wte}{w_{\mathsf{te}}}
Unlike machine learning applications where the sampling distribution is given by data, 
the underlying measure $\mu(x)$ is the Lebesgue measure over a given domain when solving PDEs.
Note that, when the domain is not bounded, we cannot sample from the measure, and thus it is necessary to introduce a sampling distribution to apply our framework.
Given a distribution $\ptr(x)$ that is supported over the support of $\mu(x)$,
the inner product between $|f\rangle$ and $|\Top f\rangle$ can be written as
\begin{align*}
\langle f | \Top f\rangle 
&= \int f(x) \Top f(x) \mu(x) \diff x \\
&= \int f(x) \sqrt{\frac{\mu(x)}{\ptr(x)}} \Top f(x) \sqrt{\frac{\mu(x)}{\ptr(x)}} \ptr(x)\diff x\\
&= \int \frac{f(x)}{\sqrt{\wtr(x)}} 
\frac{\Top f(x)}{\sqrt{\wtr(x)}} 
\ptr(x)\diff x.
\end{align*}
Here, we define the (training) importance function $\wtr(x)\defeq \frac{\ptr(x)}{\mu(x)}$.
For the case of the Lebesgue measure, one can simply regard $\mu(x)$ as 1.
It is sometimes crucial to choose a good training sampling distribution, especially for high-dimensional problems.

Suppose now that we directly parameterize $\frac{f(x)}{\sqrt{\wtr(x)}}$ by a neural network $\ft(x)$.
Then, the inner product can be computed as
\[
\langle f | \Top f\rangle 
=\int \ft(x) \frac{\Top f(x)}{\sqrt{\wtr(x)}} \ptr(x) \diff x.
\]
Here, $\Top f(x)$ can be computed by applying the operator $\Top$ to the function $x\mapsto \sqrt{\wtr(x)} \ft(x)$.

During the test phase, we may use another test distribution $\pte(x)$ to evaluate the inner product. 
Given another valid sampling distribution $\pte(x)$,
\[
\langle f | \Top f\rangle 
= \int \ft(x) \frac{\Top f(x)}{\sqrt{\wtr(x)}} \frac{\ptr(x)}{\pte(x)} \pte(x) \diff x
= \int \ft(x) \frac{\Top f(x)}{\sqrt{\wtr(x)}} \frac{\wtr(x)}{\wte(x)} \pte(x) \diff x,
\]
where we define the (test) importance function $\wte(x)\defeq \frac{\pte(x)}{\mu(x)}$.
Again, for high-dimensional problems, it is crucial to choose a good sampling distribution for reliable evaluation.

In our implementation, the operator is defined with the following interface: 
given a neural network model $\ft(x)$ and a training importance function $\pte(x)$, \texttt{operator(model, x, importance)}
outputs $\frac{\Top f(x)}{\sqrt{\wtr(x)}}$ and $\ft(x)$, so that they can be taken to the inner product directly under $\ptr(x)$.
Given this, the original function value $f(x)$ can be recovered as $f(x)=\sqrt{\wtr(x)} \ft(x)$.

For example, we implement the negative Hamiltonian as follows:
\begin{lstlisting}[]
class NegativeHamiltonian:
    def __init__(
            self,
            local_potential_ftn,
            scale_kinetic=1.,
            laplacian_eps=1e-5,
            n_particles=1
    ):
        self.laplacian_eps = laplacian_eps
        self.laplacian = VectorizedLaplacian(eps=laplacian_eps)
        self.local_potential_ftn = local_potential_ftn
        self.scale_kinetic = scale_kinetic
        self.n_particles = n_particles

    def __call__(self, f, xs, importance=None, threshold=1e5):
        # threshold is to detect an anomaly in the hamiltonian
        lap, grad, fs = self.laplacian(f, xs, importance)
        kinetic = - self.scale_kinetic * lap
        potential = self.local_potential_ftn(xs.reshape((xs.shape[0], self.n_particles, -1))).view(-1, 1) * fs
        hamiltonian = kinetic + potential
        return - hamiltonian, fs
\end{lstlisting}
Here, \texttt{VectorizedLaplacian} refers to a function for vectorized Laplacian computation, whose implementation can be found in our code.

\subsection{NestedLoRA Gradient Computation for CDK}
\label{app:sec:cdk_implementation}
Recall that the CDK is defined as $k(x,y)=\kbar(x,y)-1$ with $\kbar(x,y)\defeq\frac{p(x,y)}{p(x)p(y)}$.
We note that it is known that $\kbar(x,y)$ has the constant functions as the singular functions with singular value 1, \ie $(1, x\mapsto 1, y\mapsto 1)$ is the first singular triplet of $\kbar(x,y)$; see, \eg \citep{Huang--Makur--Wornell--Zheng2019}.
Hence, the term ``$-1$'' in the definition of CDK is to remove the first \emph{trivial} mode of $\kbar(x,y)$.

In our implementation, we handle the decomposition of CDK by considering $\kbar(x,y)$ with explicitly augmenting the constant functions as the fictitious first singular functions, so that we effectively learn from the second singular functions of $\kbar(x,y)$ and on.
For $\kbar(x,y)=\frac{p(x,y)}{p(x)p(y)}$, the ``operator term'' $\langle g_\ell | \overline{\Kop} f_\ell \rangle$ can be computed as, again by change of measure,
\begin{align*}
\langle g_\ell | \overline{\Kop} f_\ell \rangle
= \E_{p(x,y)}[f_\ell(X)g_\ell(Y)].
\end{align*}
Hence, compared to \eqref{eq:nesting_grad_evd},
the gradient becomes, for each $\ell=1,\ldots,L$,
\begin{align}
(\partial_{f_\ell} \Lc)(x_b)
=2\Bigl\{-\mm_\ell g_\ell(y_b)
+\sum_{i=0}^L \matrixmask_{i\ell} f_i(x_b) \langle g_i| g_\ell\rangle  \Bigr\}
\label{eq:nesting_grad_cdk}
\end{align}
where we set $f_0(x)\equiv 1$ and $g_0(y)\equiv 1$;
$(\partial_{g_\ell} \Lc)(x_b)$ is similarly computed.
The following snippet implements this gradient using a custom gradient as before.
Note that the constant 1's are explicitly appended as the first mode in line 16-18.

\begin{lstlisting}[]
class NestedLoRALossFunctionForCDK(torch.autograd.Function):
    @staticmethod
    @torch.cuda.amp.custom_fwd
    def forward(
            ctx: torch.autograd.function.FunctionCtx,
            f,
            g,
            vector_mask,
            matrix_mask,
    ):
        """
        the reduction assumed here is `mean` (i.e., we take average over batch)
            f: (B, L)
            g: (B, L)
        """
        pad = nn.ConstantPad1d((1, 0), 1)
        f = pad(f)
        g = pad(g)
        ctx.vector_mask = vector_mask = vector_mask.to(f.device)
        ctx.matrix_mask = matrix_mask = matrix_mask.to(f.device)
        loss_metric, lam_f, lam_g = compute_loss_metric(f, g, matrix_mask)
        ctx.save_for_backward(f, g, lam_f, lam_g)
        # compute loss_operator = -2 * E_{p(x,y)}[f^T(x) g(y)]
        loss_operator = - 2 * torch.einsum('l,bl,bl->b', vector_mask, f, g).mean()  # O(B1 * L)
        loss = loss_operator + loss_metric
        gram_matrix = f @ g.T  # (B, B); each entry is (f^T(x_i) g(y_j))
        rs_joint = gram_matrix.diag()
        rs_indep = off_diagonal(gram_matrix)
        return loss, loss_operator, loss_metric, rs_joint, rs_indep

    @staticmethod
    @torch.cuda.amp.custom_bwd
    def backward(
            ctx: torch.autograd.function.FunctionCtx,
            grad_output: torch.Tensor,
            *args
    ) -> Tuple[torch.Tensor, ...]:
        """
        Args:
            ctx: The context object to retrieve saved tensors
            grad_output: The gradient of the loss with respect to the output
        """
        f, g, lam_f, lam_g = ctx.saved_tensors
        # for grad(f)
        operator_f = - (2 / f.shape[0]) * torch.einsum('l,bl->bl', ctx.vector_mask, g)
        metric_f = (2 / f.shape[0]) * torch.einsum('il,il,bi->bl', ctx.matrix_mask, lam_g, f)
        grad_f = operator_f + metric_f
        # for grad(g)
        operator_g = - (2 / g.shape[0]) * torch.einsum('l,bl->bl', ctx.vector_mask, f)
        metric_g = (2 / g.shape[0]) * torch.einsum('il,il,bi->bl', ctx.matrix_mask, lam_f, g)
        grad_g = operator_g + metric_g
        grad_f = grad_f[:, 1:]
        grad_g = grad_g[:, 1:]
        return grad_output * grad_f, grad_output * grad_g, None, None, None, None
\end{lstlisting}

In practice, given minibatch samples $\{(x_b, y_b)\}_{b=1}^B$ drawn from a joint distribution $p(x,y)$, we can compute $\{(\fv(x_b), \gv(y_b))\}_{b=1}^B$ and plug in to the function above as follows. 
\begin{lstlisting}[]
def compute_loss(
        f,
        g,
) -> torch.Tensor:
    return NestedLoRALossFunctionForCDK.apply(
        f,
        g,
        vector_mask,
        matrix_mask,
    )
\end{lstlisting}
Here, \texttt{vector\_mask} and \texttt{matrix\_mask} should be computed using the mask computing functions in \S\ref{app:sec:helper} with \texttt{set\_first\_mode\_const=True}.

\subsection{Spectrum Estimation via Norm Estimation}\label{app:sec:norms}
As alluded to in the main text, we can estimate the singular values $(\sigma_1,\ldots,\sigma_L)$ from the learned functions and training data, \ie
\begin{align*}
\sigma_\ell=\sqrt{\E_{p(x)}[(f_\ell^\star(X))^2] \E_{p(y)}[(g_\ell^\star(Y))^2]}  \quad\text{ for }\ell=1,\ldots, L.
\end{align*}
Here, by replacing the expectation with the empirical expectation and the optimal $f_\ell^\star, g_\ell^\star$ with the learned ones $\hat{f}_\ell,\hat{g}_\ell$, 
we obtain the singular value estimator:
\[
\hat{\sigma}_\ell
\defeq \sqrt{\E_{\ph(x)}[(\hat{f}_\ell(X))^2] \E_{\ph(y)}[(\hat{g}_\ell(Y))^2]}  \quad\text{ for }\ell=1,\ldots, L.
\]
\begin{lstlisting}[]
def singular_values(f_t, g_t):  # f_t: (M, L); g_t: (N, L)
    return ((f_t ** 2).mean(dim=0) * (g_t ** 2).mean(dim=0)).sqrt()  # (L, )
\end{lstlisting}
For symmetric, PD kernels and operators, the eigenvalue estimator becomes:
\[
\hat{\lambda}_\ell
\defeq \E_{\ph(x)}[(\hat{f}_\ell(X))^2] \quad\text{ for }\ell=1,\ldots, L.
\]

\subsection{Sequential Nesting for Shared parameterization}
\label{app:sec:seq_shared}
As alluded to earlier in footnote~\ref{footnote:seq_nesting}, we can still apply sequential nesting even when the functions $\{(f_\ell,g_\ell)\}_{\ell=1}^L$ are parameterized by a shared model with a collective parameter $\th$. 
The idea is to consider a \emph{masked} gradient $\tilde{\partial}_\th (\loraobj)_{\ell}$,
which is a masked version of the original gradient ${\partial}_\th (\loraobj)_{\ell}$
computed with the assumption that $|\partial_{f_{\ell'}}(\loraobj)_{\ell}\rangle=0$ and $|\partial_{g_{\ell'}}(\loraobj)_{\ell}\rangle=0$ for every $1\le \ell'<\ell$, for each $\ell$.
The resulting masked gradient can be explicitly written as
\begin{align*}
\tilde{\partial}_\th (\loraobj)_{\ell}
&=\sum_{\ell=1}^L \{\langle\partial_\th f_\ell | \partial_{f_\ell} (\loraobj)_{\ell}\rangle 
+ \langle \partial_\th g_\ell | \partial_{g_\ell} (\loraobj)_{\ell}\rangle\}.
\end{align*}

\section{Experiment Details}
\label{app:sec:exp}
In this section, we provide all the details for our experiments.
All experiments were run on a single GPU (NVIDIA GeForce RTX 3090).
\ifforreview
Codes and scripts to replicate the experiments can be found in Supplementary Material.
\else
Codes and scripts to replicate the experiments have been open-sourced online.\footnote{\url{https://github.com/jongharyu/neural-svd}}
\fi 

\subsection{Solving Time-Independent Schr\"odinger Equations}

\subsubsection{2D Hydrogen Atom}
\label{app:sec:hydrogen}

\paragraph{Analytical Solution.}
For the 2D-confined hydrogen-like atom, the Hamiltonian is given as $\Hc=\Top+\Vop=-\frac{\hslash^2}{2m}\nabla^2 - \frac{Ze^2}{\|\xb\|_2}$, where $Z$ is the charge of the nucleus.
\citet{Yang--Guo--Chan--Wong--Ching1991} provides a closed-form expression of the eigenfunctions for this special case.
Here, we present the formula with slight modifications for visualization purposes.

By normalizing constants (\ie $Ze^2\gets 1, \frac{2m}{\hslash^2}\gets 1$), we can simplify it to the eigenvalue problem $(\nabla^2 + \frac{1}{\|\xv\|_2})\psi(\xv) = \lambda \psi(\xv)$ for $\xv\in\Real^2$.
Each eigenstate is parameterized by a pair of integers $(n,l)$ for $n\ge 0$ and $-n\le l \le n$, where the (negative) eigenenergy is $\lambda_{n,l}\defeq (2n+1)^{-2}$. Note that for each $n\ge 0$, there exist $2n+1$ degenerate states that have the same energy.
Further, the operator is PD and compact, since $\lambda_{n,l}>0$ and the Hilbert--Schmidt norm of the negative Hamiltonian is finite, \ie $\sum_{n\ge 0}\sum_{l=-n}^n \lambda_{n,l}^2=\sum_{n\ge 0}\frac{1}{(2n+1)^3}<\infty$.

The eigenfunctions can be explicitly expressed in the spherical coordinate system
\begin{align}
\label{eq:hydrogen_eigenfunction_analytical}
\psi_{n,l}(\xv)=\psi_{n,l}(r,\th)\defeq \psi_{n,l}(r)\psi_l(\th), 
\end{align}
where the radial part is
\[
\psi_{n,l}(r) = \frac{\b_n}{(2|l|)!} \Bigl(\frac{(n+|l|)!}{(2n+1)(n-|l|)!}\Bigr)^{\half} 
(\b_n r)^{|l|} e^{-\frac{\b_n r}{2}} {}_1F_1(-n+|l|, 2|l|+1, \b_n r)
\]
with $\b_n=(n+\half)^{-1}$,
and the angular part is
\[
\psi_l(\th)
=\begin{cases}
\frac{1}{\sqrt{\pi}} \cos(l \th) & \text{if }l >0\\
\frac{1}{\sqrt{2\pi}} & \text{if }l=0\\
\frac{1}{\sqrt{\pi}} \sin(l \th) & \text{if }l <0.
\end{cases}
\]
Here, ${}_1F_1(a;b;x)$ denotes the confluent hypergeometric function.

\paragraph{Implementation Details.}
We adopted the training setup of \citep{Pfau--Petersen--Agarwal--Barrett--Stachenfeld2019} with some variations.
\begin{itemize}
\item \textbf{Differential operator.} To reduce the overall complexity of the optimization, we approximated the Laplacian by the standard finite difference approximation: for $\eps>0$ sufficiently small,
\[
\nabla^2 f(\xb)
\approx \frac{1}{\eps^2} \sum_{i=1}^D 
(f(\xb+\eps\eb_i)+f(\xb-\eps\eb_i)-2f(\xb)).
\]
In this paper, we used $\eps=0.01$ throughout.

\item \textbf{Sampling distribution.}
We chose a sampling distribution $\ptr(x)$ as a Gaussian distribution $\Nc(0,16^2\Imatrix_2)$; see \S\ref{app:sec:importance_sampling}.

\item \textbf{Architecture.}
We used 16 disjoint three-layer MLPs with 128 hidden units to learn the first $L=16$ eigenfunctions, except $L=9$ for SpIN that did not fit to a single GPU due to the large memory requirement; see \S\ref{app:sec:spin}.
For the nonlinear activation function, we used the softplus activation $f(x)=\log(1+e^x)$
following the implementation of \citep{Pfau--Petersen--Agarwal--Barrett--Stachenfeld2019}.
We also found that multi-scale Fourier features~\citep{Wu--Wang--Perdikaris2023} are effective, especially the non-differentiable points at the origin for some eigenstates of the 2D hydrogen atom.
The multi-scale Fourier feature is defined as follows.
Let $D=2$ denote the input dimension. 
For $K\in\Natural$ and $\kappa>0$, we initialize and fix a Gaussian random matrix $\Bm\in\Real^{K\times D}$, each of which entry is drawn from $\Nc(0,2\pi \kappa)$.
An input is projected by $\Bm$ to the $K$ dimensional space, and mapped into Fourier features $(\cos(\Bm \xv), \sin(\Bm\xv)) \in \Real^{2K}$, following \citep{Tancik--Srinivasan--Mildenhall--Fridovich-Keil--Raghavan--Singhal--Ramamoorthi--Barron--Ng2020}.
In our experiments, we also appended the raw input $\xv$ to the Fourier feature, so that the feature dimension becomes $2K+D$.
We used $K=1024$ for NeuralEF and NeuralSVD, and $K=512$ for SpIN. Lastly, $\kappa=0.1$ was used.

\item \textbf{Optimization.}
We trained the networks for $5\times 10^5$ iterations with batch size 128 and 512.
For all methods, we used the RMSProp optimizer~\citep{Hinton--Srivastava--Swersky2012} with learning rate $10^{-4}$ and the cosine learning rate schedule~\citep{Loshchilov--Hutter2016}. 

\item \textbf{Evaluation.}
During the evaluation, we applied the exponential moving average (over the model parameters) with a decay rate of $0.995$ for smoother results.
We also used a uniform distribution over $[-100,100]^2$ as a sampling distribution, assuming that the eigenfunctions vanish outside the box, which is approximately true.
\S\ref{app:sec:importance_sampling} for the detailed procedure for the importance sampling during evaluation.

\end{itemize}

\subsubsection{2D Harmonic Oscillator}
\label{app:sec:harmonic_oscillator}

\paragraph{Analytical Solution.}
Define 

\[
\psi _{n}(x)={\frac {1}{\sqrt {2^{n}\,n!}}}\left({\frac {b}{\pi }}\right)^{1/4}e^{-{\frac {b x^{2}}{2}}}H_{n}({\sqrt {b}}x),\qquad n=0,1,2,\ldots.
\]
Here, $H_n(x)$ denotes 
the physicists' Hermite polynomials
\[
H_{n}(z)=(-1)^{n}~e^{z^{2}}{\frac {d^{n}}{dz^{n}}}(e^{-z^{2}}),
\]
and we simplify the constant $b=\frac{m\omega}{\hslash}$ to 1.
Then $\{\psi_n(x)\}_{n\ge 0}$ characterizes the eigenbasis of 1D harmonic oscillator.

Each eigenstate of the 2D harmonic oscillator is characterized by a pair of nonnegative integers $(n_x,n_y)$, and $\lambda_{n_x,n_y}=2(n_x+n_y+1)$, where a canonical representation of the eigenfunction is
\[
\psi_{n_x,n_y}(x,y)\defeq  \psi_{n_x}(x)\psi_{n_y}(y).
\]
Note that for each $n\ge 0$, there exist $n+1$ eigenstates that share the same eigenvalue $2(n+1)$.

\paragraph{Implementation Details.}
We used an almost identical setup to the 2D hydrogen atom experiment except the followings.
\begin{itemize}
\item \textbf{Operator shifting.} We chose to decompose $\Top+c\Iop$ for $c=16$, so that the first 28 eigenstates have positive eigenvalues.
\item \textbf{Sampling distribution.}
We chose a sampling distribution $\ptr(x)$ as a Gaussian distribution $\Nc(0,4^2\Imatrix_2)$.
\item \textbf{Architecture.} We used the same disjoint parameterization as before, but with $K=256$ and $\kappa=1$.

\item \textbf{Optimization.}
We trained the networks for $10^5$ iterations with batch size 128 and 512.

\item \textbf{Evaluation.}
We also used a uniform distribution over $[-5,5]^2$ as a sampling distribution.

\end{itemize}

\begin{figure*}[t!]
    \centering
    \subfloat[2D hydrogen atom (16 eigenstates)]{%
    \includegraphics[width=\textwidth,valign=t]{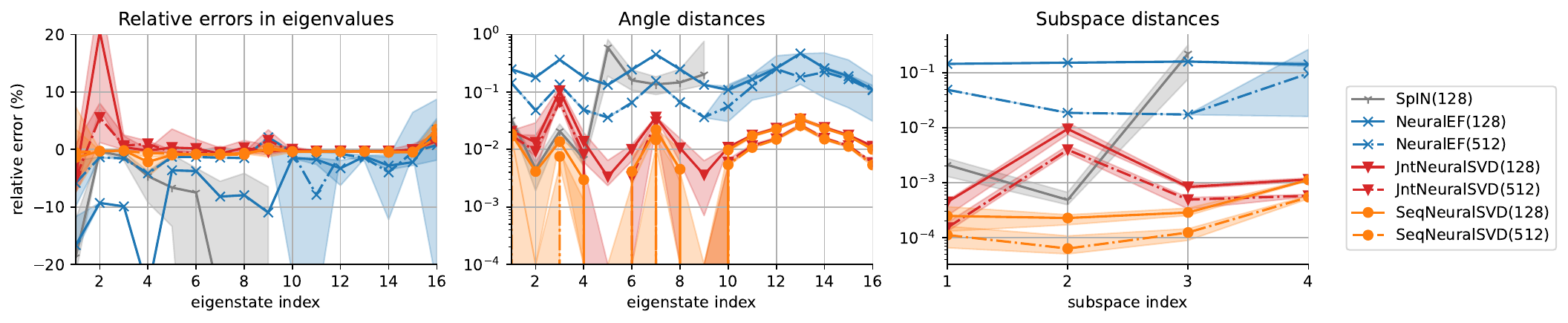}
    }\\
    \subfloat[2D harmonic oscillator (only 28 positive eigenstates out of 55 states are to be learned)]{%
    \includegraphics[width=\textwidth,valign=t]{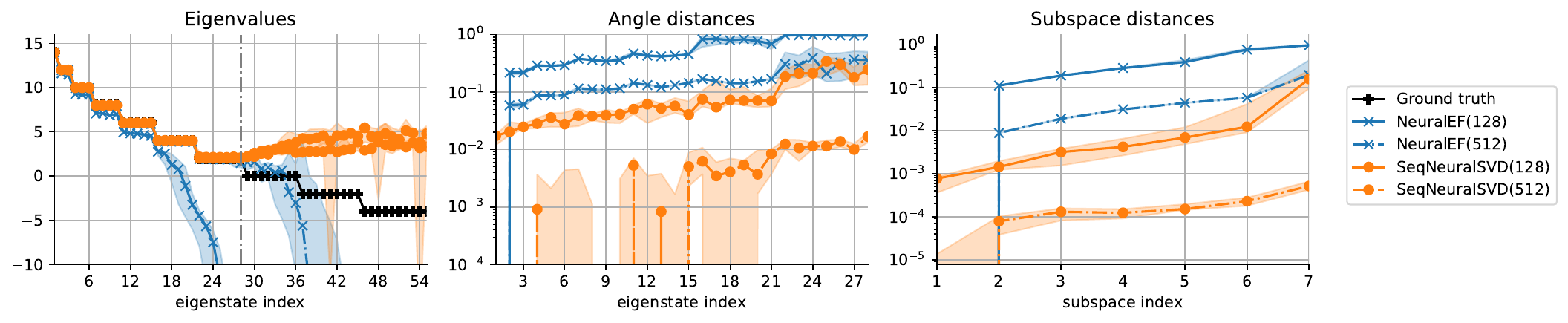}
    }
    \caption{Quantitative evaluations of the learned eigenfunctions. The shaded region indicates 20\% and 80\% quantiles with respect to 10 random seeds. In the first panel of (b), the left of the black vertical line indicates the positive eigenvalues.}
    \label{fig:pde_summary}
\end{figure*}

\begin{remark}[On shifting]
We note that the shifting technique can be applied to similar non-compact operators in general, but the shifting parameter $c$ needs to be tuned by trial and error as the underlying spectrum is unknown in practice. 
But how should one choose the parameter?
Since any $c$ beyond a certain threshold makes the first $L$ modes with strictly positive eigenvalues for a fixed number of modes $L$, one may ask whether using larger $c$ is always a safe choice.
On one hand, if $c$ is too large, the shifted operator $\Top+c\Iop$ is dominated by the identity operator, which admits \emph{any} set of orthonormal functions as its orthonormal eigenbasis.
On the other hand, $c$ needs to be sufficiently large to ensure the $L$-th mode to be recovered.
Therefore, in practice, $c$ needs to be tuned as a hyperparameter considering such a trade-off.
\end{remark}

\subsubsection{Definitions of Reported Measures}
\label{app:sec:def_measures}
\newcommand{\Sm}{\mathsf{S}}

We first provide the definitions of the reported measures in Fig.~\ref{fig:pde_summary_bar_full} and Fig.~\ref{fig:pde_summary}.
\begin{itemize}
\item \textbf{Relative errors in eigenvalues}: Given a learned eigenfunction $\psit_\ell(x)$, we estimate the learned eigenvalue by the Rayleigh quotient
\[
\tilde{\lambda}_\ell\defeq \frac{\langle \psit_\ell|\Top\psit_\ell\rangle}{\langle \psit_\ell|\psit_\ell\rangle},
\]
where each inner product is computed by importance sampling with finite samples from a given sampling distribution; see \S\ref{app:sec:importance_sampling}.
For each $\ell$, we then report the absolute relative error
\[
\frac{(\tilde{\lambda}_\ell - \lambda_\ell)}{\lambda_\ell} \times 100~(\%),
\]
for $\lambda_\ell >0$.
\item \textbf{Angle distances}:
When there is degeneracy, \ie several eigenstates share same eigenvalue, we need to \emph{align} the learned functions within each subspace before we evaluate the performance eigenstate-wise.
For such an alignment, we use the orthogonal Procrustes (OP) procedure defined as follows.
Suppose that $\Am\in\Real^{N\times K}$ and $\Bm\in\Real^{N\times K}$ are given. We wish to find the find the orthogonal transformation $\Am\mathsf{\Omega}$ that best approximates the reference $\Bm$.
The OP procedure defines the optimal $\mathsf{\Omega}^\star$ by the optimization problem
\[
{\displaystyle {\begin{aligned}{\underset {\mathsf{\Omega}\in\Real^{K\times K} }{\text{minimize}}}\quad &\| \Am\mathsf{\Omega}-\Bm\|_{F}\\{\text{subject to}}\quad &\mathsf{\Omega}^{\intercal}\mathsf{\Omega} =I.
\end{aligned}}}
\]
The solution is characterized by the SVD of $\Am^\intercal\Bm\in\Real^{K\times K}$. If $\Am^\intercal\Bm=\Um\Sm\Vm^\intercal$ is the SVD, then $\mathsf{\Omega}^\star=\Um\Vm^\intercal$.
In our case, $\Am$ is the vertical stack of the learned eigenfunctions and $\Bm$ is that of the ground truth eigenfunctions that correspond to a degenerate eigensubspace.
Here, $K$ is the number of degeneracy and 
$N$ is the number of points used for the alignment.

Given the aligned learned function $\bar{\psi}_\ell(x)$, we report the normalized angle distance
\[
\angle (|\bar{\psi}_\ell\rangle,|{\psi}_\ell\rangle)
\defeq \frac{2}{\pi} \arccos |\langle \psi_\ell|\bar{\psi}_\ell\rangle|\in[0,1].
\]
Here, we assume that both $|\bar{\psi}_\ell\rangle$ and $|{\psi}_\ell\rangle$ are normalized.

\item \textbf{Subspace distances}:
Another standard quantitative measure is the \emph{subspace distance} defined as follows.
Given $\Am\in\Real^{N\times K}$ and $\Bm\in\Real^{N\times K}$, the normalized subspace distance between the column subspaces of the two matrices is defined as 
\[
d(\Am,\Bm)\defeq 1 - \frac{1}{K}\tr(\Pm\Qm),
\]
where $\Pm=\Am(\Am^\intercal\Am)^{-1}\Am^\intercal\in\Real^{N\times N}$ and $\Qm=\Bm(\Bm^\intercal\Bm)^{-1}\Bm^\intercal\in\Real^{N\times N}$ are the projection matrices onto the column subspaces of $\Am$ and $\Bm$, respectively.
We note that $\Am$ and $\Bm$ correspond to the learned and ground truth eigenfunctions that correspond to a given subspace as above.
\end{itemize}

The reported measures in Fig.~\ref{fig:pde_summary_bar_full} are \emph{averaged} versions of the quantities defined above, except the orthogonality.

\begin{itemize}
\item \textbf{Relative errors in eigenvalues}: Report the average of the absolute relative errors over the eigenstates.
\item \textbf{Angle distance}:
Report the average of the angle distances over the eigenstates.
\item \textbf{Subspace distance}:
Report the average of the subspace distances over the degenerate subspaces.
\item \textbf{Orthogonality}:
To measure the orthogonality of the learned eigenfunctions, we report
\[
\frac{1}{N^2} \sum_{\ell=1}^L\sum_{\ell'=1}^L (\langle \tilde{\psi}_\ell|\tilde{\psi}_{\ell'}\rangle - \d_{\ell\ell'})^2.
\]
\end{itemize}

\subsection{Cross-Domain Retrieval with Canonical Dependence Kernel}
\label{app:sec:cdk_exp_detail}
We used the Sketchy Extended dataset~\citep{Sketchy2016,SketchyExtended2017} to train and evaluate our framework.
There are total 75,479 sketches ($x$) and 73,002 photos ($y$) from 125 different classes. 

We followed the standard training setup in the literature~\citep{Hwang--Kim--Hong--Kim2020IIAE}.

\begin{itemize}
\item \textbf{Sampling distribution.}
As described in the main text, we define a sampling distribution as follows.
First, note that we are given (empirical) class-conditional distributions $\{p(x|c)\}_{c=1}^K$ and $\{p(y|c)\}_{c=1}^K$ for each class $c\in[K]$.
Given the (empirical) class distribution $p(c)$, we define the joint distribution
\[
p(x,y)\defeq \E_{p(c)}[p(x|C)p(y|C)].
\]
That is, in practice, to draw a sample from $p(x,y)$, we can draw $C\sim p(c)$, and draw $(X,Y)\sim p(x|C)p(y|C)$.

\item \textbf{Pretrained fetures.}
We used a pretrained VGG16 network~\citep{Simonyan--Zisserman2014VGG} to extract features of the sketches and images. 
The pretrained VGG network and train-test splits for evaluation are from the codebase\footnote{\url{https://github.com/AnjanDutta/sem-pcyc}} of \cite{Dutta--Akata2019SEM-PCYC}.
Hence, each sketch and photo is represented by a 512-dim. feature from the VGG network.

\item \textbf{Architecture.} 
Treating the 512-dim. pretrained features as input, we used a single one-layer MLP of 8192 hidden units whose output dimension is 512. 
At the end of the network, we regularized the output so that the norm of the output has $\ell_2$-norm less than equal to $\mu=16$, \ie $\|\fv(x)\|_2\le \mu$ for every $x$.

\item \textbf{Optimization.}
We trained the network for 10 epochs with batch size of 4096.
We used the SGD optimizer with learning rate $5\times 10^{-3}$ and momentum $0.9$, together with the cosine learning rate schedule~\citep{Loshchilov--Hutter2016}.

\end{itemize}

\paragraph{Evaluation Metrics.}
Precision@$k$ and mean average precision are widely used metrics for evaluating a retrieval system such as search engines~\citep{Salton--McGill1986}. 
When $k$ items are retrieved for a query, Precision@$k$ (P@$k$) is defined as the number of relevant items (i.e., the number of photos of the same class as a query sketch in our scenario) divided by $k$. Average precision (AP) is also defined for a certain query point. When there are $n$ photos in the candidate pool, the AP is defined as $\sum_{k=1}^n P(k) \times (R(k) - R(k-1))$, where $P(k)$ denotes P@$k$ and $R(k)$ denotes Recall@k, which is defined as the number of relevant items in the $k$ retrievals divided by the total number of relevant items (\ie, number of ``all'' photos of the same class as a query sketch). Then, finally, the mean average precision (mAP) is defined as the average of all average precision over all possible queries.